\documentclass[review]{elsarticle}

\usepackage{amsmath,amssymb,amsfonts}
\usepackage{mathtools}
\usepackage{amsthm}
\usepackage{multirow}
\usepackage{graphicx}
\graphicspath{{./figures/}}
\usepackage{booktabs}
\usepackage{float}
\usepackage{subcaption}
\usepackage{comment}
\usepackage{algorithm}
\usepackage{algpseudocode}
\usepackage{multirow}

\usepackage[a4paper,margin=20mm]{geometry}

\theoremstyle{plain}
\newtheorem{theorem}{Theorem}[section]
\newtheorem{lemma}[theorem]{Lemma}

\theoremstyle{definition}
\newtheorem{definition}{Definition}[section]

\theoremstyle{remark}

\usepackage[hidelinks]{hyperref}

\usepackage[nameinlink,noabbrev]{cleveref}

\journal{Expert Systems with Applications}

\biboptions{authoryear}

\newcommand{\R}{\mathbb{R}}
\newcommand{\cL}{\mathcal{L}}
\newcommand{\cN}{\mathcal{N}}
\newcommand{\norm}[1]{\left\lVert #1 \right\rVert}

\begin{document}
\begin{frontmatter}

\title{When Gradient Clipping Becomes a Control Mechanism for Differential Privacy in Deep Learning}

\author[label1]{Mohammad Partohaghighi}
\ead{mpartohaghighi@ucmerced.edu}

\author[label2]{Roummel Marcia}
\ead{rmarcia@ucmerced.edu}

\author[label3]{Bruce J. West}
\ead{brucejwest213@gmail.com}

\author[label4]{YangQuan Chen\corref{cor1}}
\ead{ychen53@ucmerced.edu}

\cortext[cor1]{Corresponding author.}

\address[label1]{%
Electrical Engineering and Computer Science,\\
University of California, Merced, CA 95343, USA
}

\address[label2]{%
Department of Applied Mathematics,\\
University of California, Merced, CA 95343, USA
}

\address[label3]{%
Department of Innovation and Research,\\
North Carolina State University, Raleigh, NC, USA
}

\address[label4]{%
Mechatronics, Embedded Systems and Automation (MESA) Lab,\\
Department of Mechanical Engineering, School of Engineering,\\
University of California, Merced, CA 95343, USA
}

\begin{abstract}
Privacy-preserving training of neural networks for sensitive data analysis often relies on differentially private stochastic optimization with gradient clipping and noise injection. In this setting, the clipping threshold is a critical control knob: a small threshold can introduce substantial optimization bias through systematic over-clipping, while a large threshold can amplify the effect of injected noise and reduce predictive accuracy. Many adaptive clipping strategies attempt to tune this threshold using per-example gradient norm statistics, which can increase computational overhead and may be sensitive to dataset and architecture choices.

We propose a control-driven clipping strategy that regulates the clipping threshold using a lightweight, weight-only spectral diagnostic computed from model parameters. At periodic probe steps, the method analyzes a designated weight matrix through spectral decomposition and estimates a heavy-tailed spectral indicator that reflects training stability. This indicator is smoothed over time and fed into a bounded feedback controller that updates the clipping threshold multiplicatively in the log domain. Because the controller uses only model parameters produced during privacy-preserving training, the resulting updates constitute post-processing and do not increase privacy loss beyond that of the underlying privacy-preserving optimization mechanism under standard privacy accounting assumptions.

Across vision and tabular learning benchmarks under matched privacy budgets, the proposed approach improves utility and training stability relative to fixed-threshold clipping and remains competitive with representative adaptive clipping and privacy-preserving optimization baselines. Additional diagnostics and ablation analyses indicate stable controller behavior, robustness to reasonable probe and controller configurations, and modest computational overhead, with spectral probing accounting for only a small fraction of total training time.
\end{abstract}


\end{frontmatter}
\section{Introduction}
\label{sec:introduction}

Machine learning models are increasingly trained on sensitive data such as medical records, financial transactions, location traces, and user-generated content, where memorization and information leakage are not merely hypothetical failure modes but realistic adversarial objectives \cite{partohaghighi2025roughness}. Differential Privacy (DP) provides a rigorous, composable guarantee that bounds what an observer can infer about any single individual from the released model or training outputs \cite{dwork2006calibrating,dwork2014aodp,cheng2025adaptive,yuan2021differential}. For modern deep learning, the workhorse instantiation is \emph{Differentially Private Stochastic Gradient Descent} (DP-SGD), which clips per-example gradients and injects Gaussian noise at each step \cite{abadi2016dpsgd}. DP-SGD has enabled private training across diverse settings, but its practical performance is notoriously sensitive to one hyperparameter: the clipping threshold \(C\).

\paragraph{The clipping bottleneck in DP-SGD}
Clipping enforces a bound on per-example influence, making sensitivity finite and enabling calibrated Gaussian noise \cite{abadi2016dpsgd,dwork2014aodp}. Yet, \(C\) directly controls a three-way tension:
(i) \emph{clipping bias} (too small \(C\) truncates signal and underfits),
(ii) \emph{noise magnitude} (too large \(C\) increases the injected noise scale \(\sigma C\) for a fixed noise multiplier \(\sigma\)),
and (iii) \emph{optimization stability} (interacting with learning rates, architecture, and loss geometry).
A growing line of work has analyzed how clipping bias can dominate training dynamics and create failure modes even when privacy accounting is correct \cite{chen2020clip,xiao2023instruct}. In practice, practitioners often spend substantial compute sweeping \(C\) (and sometimes per-layer variants), which is expensive and frequently non-transferable across datasets, architectures, and training recipes.

\paragraph{Where the community has pushed: accounting, clipping, and optimizers}
On the privacy side, DP-SGD’s guarantees rely on composition and amplification-by-subsampling analyses. Rényi Differential Privacy (RDP) \cite{mironov2017rdp} and tighter subsampling accounting \cite{wang2019subsampled,balle2018subsampling} have become standard, while numerical and fast Fourier transform (FFT)-based accountants can yield tighter bounds for practical regimes \cite{gopi2021numerical,ghazi2022evolving}. Alternative formalisms such as \(f\)-DP / Gaussian DP (GDP) offer a hypothesis-testing view and often simplify composition reasoning \cite{dong2022gdp}.
On the utility side, multiple approaches aim to reduce the brittleness of the clipping choice: coordinate-wise or adaptive clipping mechanisms \cite{pichapati2019adaclip}, private quantile/median-based adaptation in user-level settings \cite{andrew2019dplac}, and more recent automatic clipping controllers for DP-SGD \cite{bu2023autoclip}. A complementary direction improves DP training stability and convergence using (private) adaptive optimization, though adaptivity itself can be privacy-costly unless carefully engineered \cite{li2023dp2,tang2024dpadambc}. Other work changes the sampling/training dynamics to improve signal-to-noise, for example importance sampling under DP constraints \cite{wei2022dpis} or sharpness-aware training adapted to DP to improve generalization under noise \cite{foret2021sam,park2023dpsat}. Finally, substantial engineering effort has gone into reducing DP-SGD overhead, especially per-example gradient computation and clipping \cite{lee2021fastclip}, and into mature software stacks for reproducible DP training \cite{yousefpour2021opacus,tensorflowprivacy2019}.

\paragraph{A spectral-control viewpoint}
This paper introduces a different lens: treat clipping selection as a \emph{closed-loop control} problem with a feedback signal that is (a) informative about training health and (b) safe under differential privacy post-processing. Our key observation is that deep networks’ weight matrices exhibit measurable spectral regularities that correlate with capacity, implicit regularization, and training phase. In particular, a line of random-matrix-inspired work argues that well-trained networks often display \emph{heavy-tailed} spectral structure and that fitted tail exponents can track model quality trends \cite{martin2021implicit,martin2019heavytail,martin2021nature}. The associated open-source \emph{WeightWatcher} tool, a weight-spectrum analysis package inspired by random matrix theory, operationalizes this idea via spectral diagnostics computed directly from weights, without accessing training or test data at probe time \cite{weightwatcher_github}.

\paragraph{A weight-spectrum feedback controller for adaptive clipping under differential privacy}
We propose \textbf{WW-DP-SGD}, a DP-SGD variant that adapts the clipping threshold \(C_t\) using a \emph{weight-only} spectral health signal computed from the current (private) model parameters. At periodic probe steps, we estimate a WeightWatcher-style heavy-tailed exponent \(\zeta_t\) from a chosen weight matrix and regulate training to keep \(\zeta_t\) within a pre-defined \emph{spectral health zone}. We use a saturated log-domain feedback law that continuously regulates toward the zone center, yielding stable two-sided updates in general and avoiding collapse of \(C_t\).
WW-DP-SGD does \emph{not} claim additional privacy beyond DP-SGD. Under the same minibatching/subsampling assumptions and the same noise multiplier \(\sigma\), privacy accounting follows the same established analyses used for DP-SGD \cite{abadi2016dpsgd,mironov2017rdp,wang2019subsampled,gopi2021numerical,dong2022gdp}. The key point is that the spectral probe and the controller read only the current model parameters, which are outputs of a DP mechanism; therefore the adaptation of \(C_t\) is post-processing and cannot increase privacy loss \cite{dwork2014aodp}. Moreover, the per-step $\ell_2$ sensitivity of the clipped gradient sum scales with \(C_t\) and the injected Gaussian noise scales with \(C_t\) as well, so the per-step privacy is governed by \((q,\sigma)\) and composed over \(T\) steps by standard RDP/GDP accountants \cite{mironov2017rdp,dong2022gdp}.

\paragraph{Contributions}
We summarize our main contributions as follows:
\textbf{(Algorithmic)} we introduce \textbf{WW-DP-SGD}, a closed-loop DP-SGD variant that adaptively sets the clipping norm using a WeightWatcher-style spectral tail exponent together with a spectral ``health zone'' \cite{martin2021implicit,weightwatcher_github};
\textbf{(Control design)} we instantiate a log-domain saturated feedback controller with interpretable parameters \((\kappa,\beta,r)\) and an optional experimental clamp \((C_{\min},C_{\max})\);
\textbf{(DP consistency)} we formalize why this controller preserves differential privacy under standard DP-SGD accounting assumptions via post-processing and composition \cite{dwork2014aodp,mironov2017rdp,wang2019subsampled};
\textbf{(Positioning vs.\ prior art)} we relate spectral feedback control to existing adaptive clipping and DP optimization methods \cite{pichapati2019adaclip,bu2023autoclip,wei2022dpis,li2023dp2,park2023dpsat};
and \textbf{(Practical guidance)} we provide implementation details such as probe-layer choice, probe period, smoothing, and controller gain along with evaluation protocols using common libraries for private training \cite{yousefpour2021opacus,tensorflowprivacy2019,lee2021fastclip}.

\paragraph{Paper organization}
Section~\ref{sec:background} reviews DP-SGD and introduces the WeightWatcher-style spectral exponent used as a weight-only diagnostic signal.
Section~\ref{sec:proposedalgorithm} presents WW-DP-SGD, including the log-domain saturated controller for adaptive clipping.
Section~\ref{sec:privacy} states the DP guarantee and the accounting protocol under Poisson subsampling.
Section~\ref{sec:experiments} evaluates privacy--utility trade-offs, robustness to clipping sensitivity and distribution shift, and runtime overhead.
Section~\ref{sec:limitations} discusses limitations and outlines directions for future work.


\section{Background and Notation}
\label{sec:background}
We first define differential privacy (DP), then instantiate DP via DP-SGD under Poisson subsampling.
Next, we introduce a \emph{weight-only spectral diagnostic} (a WeightWatcher-style proxy derived from the weight spectrum) used as the feedback signal,
and finally define the controller that adapts the clipping threshold while preserving DP guarantees through post-processing.

\subsection{Differential Privacy}
\label{subsec:dp}
\begin{definition}[Differential Privacy (DP)]\cite{abadi2016dpsgd}
\label{def:dp}
A randomized mechanism \(M\) is \((\varepsilon,\delta)\)-differentially private if for all neighboring datasets
\(D,D'\) differing in exactly one example and all measurable sets \(S\),
\[
\Pr[M(D)\in S] \le e^{\varepsilon}\Pr[M(D')\in S] + \delta.
\]
\end{definition}
DP provides a worst-case stability guarantee: changing one training point has a limited effect on the
distribution of the released output. We next review the standard training mechanism used in this work. Differentially Private Stochastic Gradient Descent (DP-SGD) is the standard way to train deep models under DP.
Let \(D=\{x_1,\ldots,x_N\}\), parameters \(\theta_t\in\R^d\), sampling rate \(q\in(0,1)\), and learning rate \(\eta_t\).
At step \(t\), we independently include each example with probability \(q\) (Poisson subsampling), forming minibatch \(L_t\). For each sampled example \(x_i\in L_t\), compute the per-example gradient
\begin{equation}
\label{eq:per_example_grad}
g_t(x_i)=\nabla_\theta \cL(\theta_t,x_i).
\end{equation}
We initialize the clipping schedule with a user-specified \(C_0>0\); in standard DP-SGD, \(C_t \equiv C_0\) is fixed across steps, while in WW-DP-SGD \(C_t\) is adapted over time by the controller in Section~\ref{subsec:controller}. Clip to an \(\ell_2\) threshold \(C_t>0\),
\begin{equation}
\label{eq:clip}
\bar g_t(x_i)=\frac{g_t(x_i)}{\max\!\left(1,\frac{\norm{g_t(x_i)}}{C_t}\right)},
\end{equation}
and add Gaussian noise calibrated to \(C_t\). Writing \(I_d\) for the \(d\times d\) identity matrix,
\begin{equation}
\label{eq:dpsgd}
\tilde g_t=
\frac{1}{|L_t|}
\left(
\sum_{x_i\in L_t}\bar g_t(x_i) + \cN\!\left(0,\sigma^2 C_t^2 I_d\right)
\right),
\qquad
\theta_{t+1}=\theta_t-\eta_t \tilde g_t,
\end{equation}
where \(\sigma\) is the noise multiplier (fixed across steps in our experiments). Under Poisson subsampling, \(|L_t|=0\) can occur. In that case, we perform a safe DP-preserving update by
skipping the step (i.e., set \(\theta_{t+1}=\theta_t\)) and retaining \(C_t\). This is post-processing and does not
weaken DP. We compose privacy over \(T\) steps using either a R\'enyi Differential Privacy (RDP) accountant or a
Privacy Loss Random Variable (PRV) accountant, both compatible with Poisson subsampling.
The accountant takes \((q,\sigma,T,\delta)\) and returns \((\varepsilon,\delta)\) for the full run. The next ingredient is a weight-only spectral proxy that we use to adapt \(C_t\) in a controlled way.


\section{The proposed algorithm}
\label{sec:proposedalgorithm}
\paragraph{Overview.}
This section presents WW-DP-SGD as a \emph{DP-safe, closed-loop} variant of DP-SGD that adapts the clipping threshold without accessing per-example gradient statistics for control.
We first define the weight-only diagnostic signal: a WeightWatcher-style estimate of a heavy-tailed spectral exponent \(\zeta_t\) computed from a fixed probe weight matrix \(W(\theta_t)\).
We then introduce the notation used throughout the method and specify a practical \emph{spectral health zone} \((\zeta_{\min},\zeta_{\max})\) that serves as the control target.
Finally, we describe a log-domain saturated feedback controller that updates the clipping threshold \(C_t\) multiplicatively (with an optional clamp for numerical stability), and we summarize the full procedure in Algorithm~\ref{alg:wwdpsgd_v8}.
\subsection{WeightWatcher-style spectral exponent (WW)}
\label{subsec:ww}
To obtain a diagnostic signal from the model weights, we apply a WeightWatcher-style (WW) heavy-tail analysis
to a fixed probe weight matrix \(W(\theta_t)\in\R^{m\times n}\) (dense layers directly; convolutional kernels reshaped to 2D).
Compute the singular value decomposition (SVD):
\begin{equation}
\label{eq:svd}
W(\theta_t) = U \Sigma V^\top, \quad \Sigma = \mathrm{diag}(\sigma_1, \dots, \sigma_r),
\end{equation}
and define eigenvalues \(\lambda_j=\sigma_j^2\) of the correlation matrix \(W^\top W\).
We estimate a heavy-tailed exponent \(\zeta_t\) by fitting a power-law to the \emph{upper tail} of the spectrum:
\begin{equation}
\label{eq:powerlaw}
p(\lambda) \propto \lambda^{-\zeta_t} \quad \text{for} \quad \lambda \geq \lambda_{\min}.
\end{equation}
In practice, the fit is performed on the largest eigenvalues (e.g., top-\(k\) or those exceeding a threshold)
using robust regression in log-log coordinates.
We denote this procedure as
\[
\zeta_t = \mathrm{WW}(W(\theta_t)).
\]
 In parts of the WW literature the exponent is often denoted by \(\alpha\);
here we use \(\zeta\) consistently throughout the paper. This scalar proxy is then smoothed and fed to a controller that updates \(C_t\) multiplicatively in log-space.
\subsection{Notation}
\label{subsec:notation}
\begin{table}[H]
\centering
\caption{Notation used throughout the paper.}
\label{tab:notation}
\begin{tabular}{ll}
\toprule
Symbol & Meaning \\
\midrule
\(D=\{x_i\}_{i=1}^N\) & Dataset; \(N\) is dataset size \\
\(d\) & Parameter dimension (\(\theta_t\in\R^d\)) \\
\(q\) & Poisson subsampling rate \\
\(t=0,\dots,T-1\) & Step index; \(T\) total number of steps \\
\(C_0\) & Initial clipping threshold (\(C_t\equiv C_0\) for fixed-clipping DP-SGD) \\
\(\theta_t\) & Model parameters at step \(t\) \\
\(\eta_t\) & Learning rate at step \(t\) \\
\(L_t\) & Poisson-subsampled minibatch at step \(t\) \\
\(g_t(x_i), \bar g_t(x_i)\) & Per-example gradient; clipped gradient \\
\(C_t\) & Clipping threshold at step \(t\) \\
\(\sigma\) & Noise multiplier (fixed across steps) \\
\(I_d\) & \(d\times d\) identity matrix (noise covariance uses \(I_d\)) \\
\(\tilde g_t\) & Noisy averaged gradient released by DP-SGD \\
\(W(\theta_t)\) & Probed weight matrix used for WW diagnostic \\
\(\zeta_t\) & Raw estimated heavy-tailed exponent \(\mathrm{WW}(W(\theta_t))\) \\
\(\hat{\zeta}_t\) & EMA-smoothed exponent used for control \\
\(\zeta_\star, r\) & Health-zone center and radius (defaults \(\zeta_\star=4, r=2\)) \\
\(K\) & Probe period (compute \(\zeta\) every \(K\) steps) \\
\(\beta\) & EMA smoothing factor \\
\(\kappa\) & Controller gain \\
\(u_t\) & Log-clip state \(u_t=\log C_t\) \\
\(e_t\) & Centered error \(e_t=\hat{\zeta}_t-\zeta_\star\) \\
\(\phi_t\) & Saturated signal \(\phi_t=\mathrm{sat}(e_t/r)\) \\
\(\mathrm{sat}(\cdot)\) & Saturation: \(\mathrm{sat}(x)=\max(-1,\min(1,x))\) \\
\(\mathrm{clip}(\cdot)\) & Clamp to interval \([a,b]\): \(\mathrm{clip}(y,a,b)=\min(b,\max(a,y))\) \\
\(C_{\min},C_{\max}\) & Optional clamp bounds for numerical stability \\
\(\lambda_{\min}\) & Lower cutoff used in tail fitting for \(\zeta\) \\
\bottomrule
\end{tabular}
\end{table}
\subsubsection{Spectral health zone}
\label{subsec:zone}
We follow standard practice in the WW literature~\cite{martin2021implicit}. by fitting the tail over the largest eigenvalues, with \(\lambda_{\min}\) selected via a goodness-of-fit criterion (e.g., KS-based cutoff) or a fixed top-\(k\) rule in small-matrix regimes; the exact protocol is reported in the experimental setup for reproducibility. We define a practical "spectral health zone" for the proxy exponent:
\begin{equation}
\label{eq:zone_bounds}
(\zeta_{\min},\zeta_{\max})=(2,6),
\end{equation}
equivalently described by a center \(\zeta_\star=4\) and radius \(r=2\):
\begin{equation}
\label{eq:zone_center_radius}
2<\zeta<6 \iff |\zeta-\zeta_\star|<r.
\end{equation}
\paragraph{Design rationale (proxy signal).}
The exponent \(\zeta_t\) is a weight-only proxy for correlation structure and implicit self-regularization
as studied in the WW line of work~\cite{martin2021implicit,martin2021nature}.
We treat \((2,6)\) as an empirical operating range used to define a stable control target rather than a universal law.
Ablations verify that performance is robust around the default center/radius. This zone directly motivates the feedback controller below: we regulate \(\hat{\zeta}_t\) toward \(\zeta_\star\),
and use the signed deviation to adapt the clipping threshold \(C_t\).
                  
\subsection{Controller: log-domain saturated regulation of \(C_t\)}
\label{subsec:controller}
\paragraph{Key idea.}
Rather than updating \(C_t\) only when \(\hat{\zeta}_t\) exits the zone, we regulate continuously toward the zone center
using a saturated error. We update the clipping norm in log-space to guarantee positivity and obtain smooth multiplicative dynamics.
We compute \(\zeta\) periodically (every \(K\) steps) and apply an Exponential Moving Average (EMA):
\begin{equation}
\label{eq:ema}
\hat\zeta_{t+1}=\beta\hat\zeta_t+(1-\beta)\zeta_{t+1},\qquad \beta\in[0,1).
\end{equation}
We set the initial condition \(C_0>0\) and equivalently \(u_0=\log C_0\). Define the log-clip state
\begin{equation}
\label{eq:log_clip_state}
u_t=\log C_t,
\end{equation}
the centered error and normalized saturated signal
\begin{equation}
\label{eq:error_and_phi}
e_t=\hat{\zeta}_t-\zeta_\star,
\qquad
\phi_t=\mathrm{sat}\!\left(\frac{e_t}{r}\right),
\qquad
\mathrm{sat}(x)=\max(-1,\min(1,x)).
\end{equation}
The controller updates \(u_t\) and maps back to \(C_t\):
\begin{equation}
\label{eq:controller_update}
u_{t+1}=u_t+\kappa \phi_t,
\qquad
C_{t+1}=\exp(u_{t+1}),
\qquad \kappa>0.
\end{equation}
For numerical stability we may apply an optional clamp:
\begin{equation}
\label{eq:clamp}
C_{t+1}\leftarrow \mathrm{clip}(C_{t+1},C_{\min},C_{\max}),
\end{equation}
where \(\mathrm{clip}(y,a,b)=\min(b,\max(a,y))\).
This clamp is post-processing on the mechanism output and does not affect DP accounting; it is an engineering safeguard against pathological drift.

\paragraph{Note on Poisson subsampling (empty minibatches).}
Under Poisson subsampling, an empty minibatch event \(|L_t|=0\) is theoretically possible, but in our settings it is negligible.
For completeness, our implementation skips the optimizer step whenever \(|L_t|=0\) to avoid division by zero in the noisy averaging,
and carries forward \((\theta_t,C_t,\hat{\zeta}_t)\) unchanged. This edge-case handling does not affect DP accounting.

\begin{algorithm}[H]
\caption{WW-DP-SGD: DP-guaranteed adaptive clipping via spectral feedback in log-space}
\label{alg:wwdpsgd_v8}
\begin{algorithmic}[1]
\Require Dataset size \(N\), loss \(\cL(\theta,x)\), steps \(T\), sampling rate \(q\)
\Require Learning rates \(\{\eta_t\}\), noise multiplier \(\sigma\)
\Require Initial clipping \(C_0>0\), probe period \(K\), EMA \(\beta\)
\Require Zone center \(\zeta_\star\), radius \(r\), controller gain \(\kappa>0\)
\Require Probe map \(W(\theta)\) (chosen layer) and optional clamp bounds \(C_{\min}, C_{\max}\)
\State Initialize \(\theta_0\); set \(u_0\leftarrow \log C_0\); set \(\hat\zeta_0\leftarrow \zeta_\star\)
\For{\(t=0,1,\dots,T-1\)}
    \State Form minibatch \(L_t\) by including each example independently with probability \(q\) (Poisson subsampling)
    \State For each \(x_i\in L_t\): compute \(g_t(x_i)=\nabla_\theta\cL(\theta_t,x_i)\)
    \State Clip with current \(C_t\):
    \[
    \bar g_t(x_i)=\frac{g_t(x_i)}{\max\!\left(1,\norm{g_t(x_i)}/C_t\right)}
    \]
    \State Noise + average:
    \[
    \tilde g_t=\frac{1}{|L_t|}\left(\sum_{x_i\in L_t}\bar g_t(x_i) + \cN(0,\sigma^2 C_t^2 I_d)\right)
    \]
    \State Update: \(\theta_{t+1}=\theta_t-\eta_t \tilde g_t\)
    \State Default carry-over (no probe): \(u_{t+1}\leftarrow u_t,\; C_{t+1}\leftarrow C_t,\; \hat\zeta_{t+1}\leftarrow \hat\zeta_t\)
    \If{\((t+1)\bmod K=0\)} \Comment{probe + control step}
        \State \(\zeta_{t+1} \leftarrow \mathrm{WW}(W(\theta_{t+1}))\)
        \State \(\hat\zeta_{t+1} \leftarrow \beta\hat\zeta_t + (1-\beta)\zeta_{t+1}\)
        \State \(e_{t+1} \leftarrow \hat\zeta_{t+1}-\zeta_\star\)
        \State \(\phi_{t+1} \leftarrow \mathrm{sat}(e_{t+1}/r)\)
        \State \(u_{t+1} \leftarrow u_t+\kappa \phi_{t+1}\)
        \State \(C_{t+1} \leftarrow \exp(u_{t+1})\) \Comment{used in subsequent steps}
        \State \(C_{t+1} \leftarrow \mathrm{clip}(C_{t+1}, C_{\min}, C_{\max})\) \textbf{(optional)}
    \EndIf
\EndFor
\State Output \(\theta_T\). Compute \((\varepsilon,\delta)\) via an RDP/PRV accountant for Poisson subsampling using \((q,\sigma,T,\delta)\).
\end{algorithmic}
\end{algorithm}

A schematic overview of the WW-DP-SGD closed-loop training procedure is shown in Figure~\ref{fig:WWDPSGD}.

\begin{figure}[H]
\centering
\includegraphics[width=0.95\textwidth]{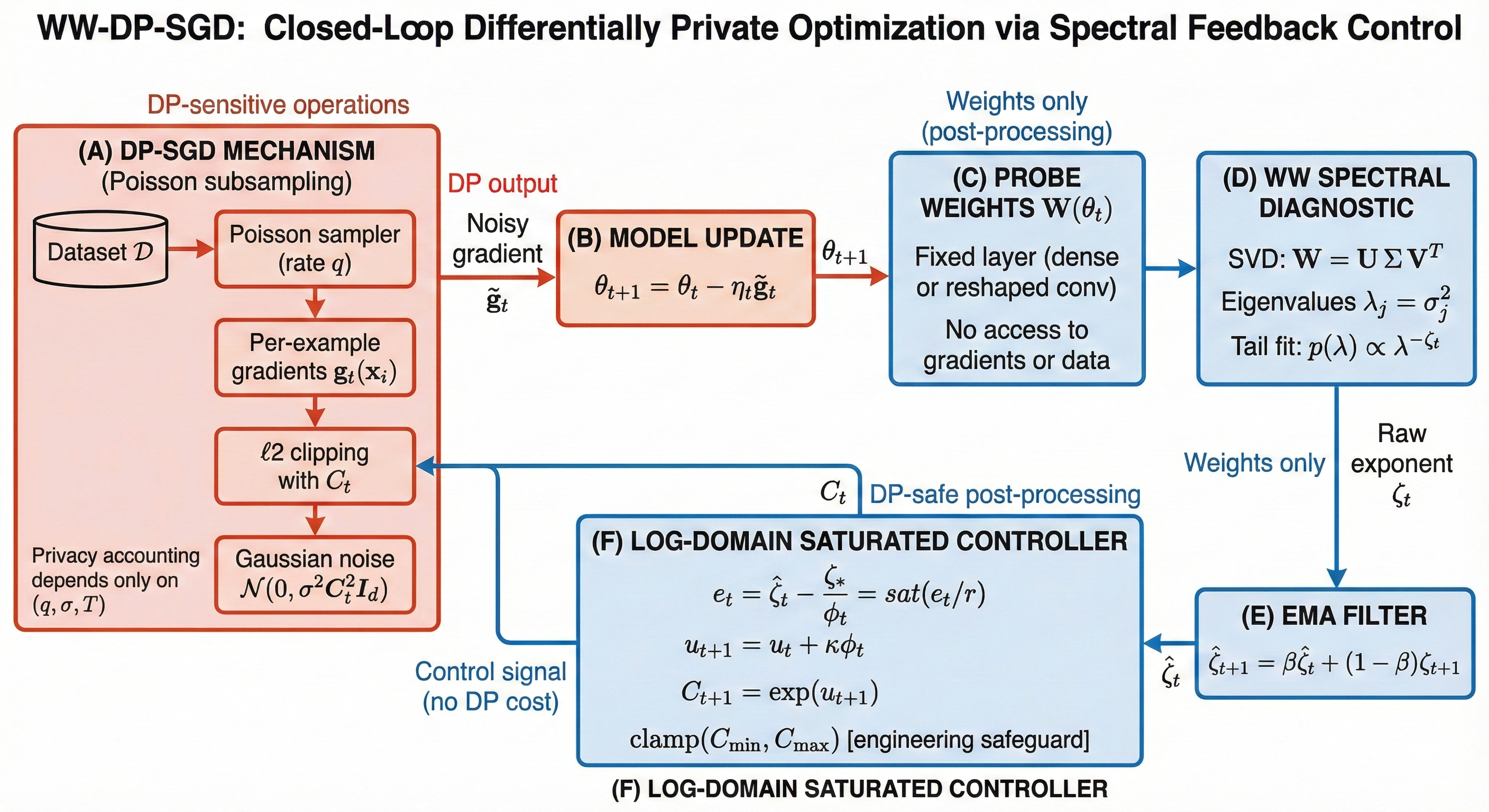}
\caption{
\textbf{WW-DP-SGD: closed-loop differentially private optimization via spectral feedback.}
DP-SGD performs Poisson subsampling, per-example gradient clipping with threshold $C_t$, and Gaussian noise injection to produce a noisy update $\tilde g_t$ with formal $(\varepsilon,\delta)$ guarantees.
A WeightWatcher-style spectral diagnostic is computed periodically from a fixed probe weight matrix $W(\theta_t)$, yielding a heavy-tailed exponent $\zeta_t$ that is smoothed and regulated toward a target spectral health zone.
The resulting control signal adaptively updates the clipping threshold in log-space.
Crucially, the feedback loop operates exclusively on model weights and released quantities, constituting post-processing and therefore preserving the original DP accounting.
}
\label{fig:WWDPSGD}
\end{figure}

\section{Privacy Analysis}
\label{sec:privacy}
\paragraph{Accounting model}
We consider standard example-level DP-SGD with fixed noise multiplier \(\sigma\), Poisson subsampling at rate \(q\),
and \(T\) total steps, as in Eq.~\eqref{eq:dpsgd}. Privacy loss is tracked via R\'enyi Differential Privacy (RDP)
and converted to \((\varepsilon,\delta)\)-DP using standard composition results
\cite{mironov2017rdp,wang2019subsampled,balle2018subsampling}, or via numerical accountants for improved tightness in
practical regimes \cite{gopi2021numerical,ghazi2022evolving}. Throughout, \(\delta\) is fixed and \(\varepsilon\) is reported.
\paragraph{Transcript and interactivity}
Let the (random) training transcript up to step \(t\) be $\mathcal{T}_t := (\theta_0,\theta_1,\ldots,\theta_t),$ i.e., the sequence of model iterates produced by the DP training procedure. WW-DP-SGD is an \emph{interactive} mechanism:
at each step \(t\), it chooses a clipping threshold \(C_t\) as a (possibly time-varying) function of the past transcript
and internal controller state, and then performs a DP-SGD update with that \(C_t\).

In our threat model, the released output is the final model \(\theta_T\) (and optionally the scalar clip schedule \(\{C_t\}\)); the full transcript \(\mathcal{T}_t\) is defined only for analysis of interactivity and adaptive composition.

\paragraph{Key observation (scale invariance w.r.t.\ \(C_t\))}
At step \(t\), DP-SGD clips per-example gradients at norm \(C_t\), so the \(\ell_2\)-sensitivity of the summed clipped
gradient is proportional to \(C_t\). The mechanism then adds Gaussian noise with standard deviation \(\sigma C_t\),
i.e., noise calibrated proportionally to the same \(C_t\) (Eq.~\eqref{eq:dpsgd}). Consequently, the per-step privacy bound
under Poisson subsampling depends on \(q\) and the \emph{noise-to-sensitivity ratio}, which is exactly the noise multiplier
\(\sigma\), and not on the realized magnitude of \(C_t\) itself, as long as noise is scaled consistently with clipping.
This is the sense in which the per-step DP guarantee is invariant to the realized value of \(C_t\)
\cite{abadi2016dpsgd,mironov2017rdp,wang2019subsampled,balle2018subsampling}. Formally, conditioning on \(\mathcal{T}_t\) fixes \(C_t\), and the step reduces to a subsampled Gaussian mechanism whose RDP/GDP bound depends on \(q\) and the noise multiplier \(\sigma\) (with noise standard deviation \(\sigma C_t\) matched to clipping \(C_t\)), but not on the realized value of \(C_t\).

\paragraph{WW diagnostic uses only DP outputs}
The controller computes a spectral proxy using only model weights from a fixed probe layer
\(W(\theta_t)\) (dense layers directly; convolutional kernels reshaped to 2D) and then forms
\(\zeta_t = \mathrm{WW}(W(\theta_t))\) (Section~\ref{subsec:ww}). Importantly, \(\mathrm{WW}(\cdot)\) accesses
\emph{only} the current model parameters (already part of the transcript) and does not access raw training data,
per-example gradients, or minibatch statistics beyond what is already privatized by DP-SGD.

Crucially, \(C_t\) is computed \emph{only} from previously released DP iterates (the transcript) and fixed hyperparameters, and never from non-private minibatch statistics (e.g., raw per-example gradient norms) at the current step; therefore the adaptation is pure post-processing of DP outputs.

\begin{lemma}[Adaptive clipping is safe]
\label{lem:adaptive_clipping_safe}
Consider an interactive mechanism that, at each step \(t\), selects \(C_t\) as an arbitrary (possibly randomized)
function of the past transcript \(\mathcal{T}_t\) and internal state, and then performs one Poisson-subsampled
Gaussian DP-SGD step with clipping threshold \(C_t\) and Gaussian noise standard deviation \(\sigma C_t\)
(as in Eq.~\eqref{eq:dpsgd}). Then the privacy cost of each step is determined by \((q,\sigma)\) (via RDP/PRV/GDP
accounting for Poisson subsampling) and does not incur additional privacy loss due to the adaptivity of \(C_t\).
\end{lemma}

\begin{proof}[Proof sketch]
Fix a step \(t\) and condition on the full past transcript \(\mathcal{T}_t\) and controller state.
Under this conditioning, \(C_t\) becomes fixed, and the step-\(t\) update is exactly a Poisson-subsampled Gaussian mechanism
applied to per-example gradients clipped at \(C_t\) with noise standard deviation \(\sigma C_t\).
Hence, the conditional per-step privacy guarantee is the standard one for subsampled Gaussian mechanisms, determined by
\((q,\sigma)\) (and the chosen order for RDP), independent of the realized value of \(C_t\)
\cite{mironov2017rdp,wang2019subsampled,balle2018subsampling}. Since \(C_t\) is computed from the past transcript,
its selection is post-processing of DP outputs. Adaptive (interactive) composition over \(T\) steps therefore yields the
same overall accountant-driven \((\varepsilon,\delta)\) guarantee as standard DP-SGD with the same \((q,\sigma,T,\delta)\).
\end{proof}

\paragraph{Adaptive control and post-processing}
In WW-DP-SGD, \(C_t\) is updated by the controller using only previously released iterates (the transcript) via the log-domain
update \(u_t=\log C_t\) and the saturated feedback law in
Eqs.~\eqref{eq:log_clip_state}--\eqref{eq:controller_update}. If used, the clamp
\(C_{t+1}\leftarrow \mathrm{clip}(C_{t+1},C_{\min},C_{\max})\) in Eq.~\eqref{eq:clamp} is also post-processing on a
transcript-derived quantity and thus does not affect DP accounting.
\begin{theorem}[DP guarantee for WW-DP-SGD]
\label{thm:dp}
Assume DP-SGD is run with Poisson subsampling rate \(q\), noise multiplier \(\sigma\), and \(T\) steps under a valid
accountant compatible with Poisson subsampling (e.g., RDP/PRV/GDP).
Let \(\theta_t\) denote the iterates produced by Eq.~\eqref{eq:dpsgd}.
Let \(C_t\) be adapted according to Eqs.~\eqref{eq:log_clip_state}--\eqref{eq:controller_update} using only the past transcript
\(\mathcal{T}_t\) (and the optional clamp in Eq.~\eqref{eq:clamp}).
Then WW-DP-SGD satisfies the same \((\varepsilon,\delta)\)-DP guarantee as standard DP-SGD under the same
\((q,\sigma,T,\delta)\).
\end{theorem}
\begin{proof}[Proof sketch]
By Lemma~\ref{lem:adaptive_clipping_safe}, each step is a valid Poisson-subsampled Gaussian mechanism conditional on the past,
with privacy cost determined by \((q,\sigma)\) and independent of the realized \(C_t\) (provided noise is scaled as \(\sigma C_t\)).
The WW-based controller uses only the transcript \(\mathcal{T}_t\) and internal state, hence is post-processing and introduces
no additional privacy loss. Applying the chosen accountant over \(T\) steps yields the stated \((\varepsilon,\delta)\) guarantee.
\end{proof}
\paragraph{Practical accounting recipe}
In all experiments, we report \(\varepsilon\) at a fixed \(\delta\), using an RDP or PRV accountant consistent with Poisson
subsampling \cite{mironov2017rdp,wang2019subsampled,balle2018subsampling,gopi2021numerical}.
We emphasize that the controller parameters \((\kappa,\beta,r,K)\) and optional clamp bounds \((C_{\min},C_{\max})\)
do not enter the privacy accountant; they influence optimization dynamics and utility but not privacy loss.

\section{Experiments}
\label{sec:experiments}
\subsection{Goals and evaluation principles}
\label{subsec:exp_goals}
We evaluate whether \textbf{WW-DP-SGD} improves utility and stability relative to standard DP-SGD by reducing sensitivity to brittle fixed clipping choices, regulating the spectral proxy toward the zone center \(\zeta_\star\) (and increasing spectral "time-in-zone" for \(\hat{\zeta}_t\)), and improving privacy--utility trade-offs under matched privacy budgets. All private methods are compared under a matched privacy protocol: the same accountant family and the same \((q,T,\delta)\); we vary the noise multiplier \(\sigma\) when sweeping \(\varepsilon\). We report mean\(\pm\)std over \(S\) random seeds when applicable.

When reporting non-IID severity using a Dirichlet partition parameter (denoted \(\alpha\)), this \(\alpha\) refers to the Dirichlet concentration parameter and is unrelated to the WW spectral exponent \(\zeta\).

\subsection{Datasets, models, and preprocessing}
\label{subsec:datasets_models}
We benchmark across vision and tabular modalities. For vision tasks, we use MNIST (10-class) with a CNN consisting of two convolutional blocks followed by two fully connected layers; CIFAR-10 (10-class) with ResNet-18 (or a smaller ResNet under constrained compute); CIFAR-100 (100-class) with ResNet-18/34 (or a smaller ResNet under constrained compute); and EMNIST with a CNN or ResNet-style architecture depending on the split (digits/letters). For tabular tasks, we use UCI Adult (binary classification) with a 3-layer MLP where categorical features are one-hot encoded and a standard train/val/test split is applied, as well as Heart Disease (binary classification) with a 2--3 layer MLP, z-score normalization, and a standard split. We further include the Vehicle Energy Dataset (VED)~\cite{oh2020vehicle}, a large-scale real-world tabular/time-series dataset for vehicle energy consumption prediction (regression task) collected from 383 vehicles, using a 3--5 layer MLP with standard preprocessing (e.g., normalization of continuous features and handling of time-series segments).
For reproducibility, we run each configuration with \(S\) random seeds (e.g., \(S=3\) or \(5\)) and report mean\(\pm\)std. We use the same data pipeline, initialization, and learning-rate schedule across methods within each dataset/model. For the MNIST experiments, we additionally report exact per-epoch \(\varepsilon\) values (from the accountant) and per-epoch test accuracy to match execution logs. For ImageNet, we use ImageNet-100 with and report Top-1 accuracy.
\subsection{Methods compared (baselines)}
\label{subsec:baselines}
We compare WW-DP-SGD against a diverse set of private and non-private baselines, covering fixed clipping, recent adaptive clipping, optimizer-level adaptations, and robustness-oriented training variants. The methods include non-private training, DP-SGD~\cite{abadi2016dpsgd}, Automatic Clipping~\cite{bu2023automatic}, DP-PSAC~\cite{xia2023dppsac}, Bounded Adaptive~\cite{zhao2025bounded}, AdaDPIGU~\cite{zhang2025adadpigu}, DP-Adam~\cite{tang2024dpadambc}, ADP-AdamW~\cite{chilukoti2025differentially},for importance-sampling baselines, we use DPIS-style sampling while matching \((q,T,\delta)\) and accounting under Poisson subsampling, DP-IS~\cite{wei2022dpis}, DP-SAM~\cite{park2023dp_sam}, DP-SAT~\cite{park2023differentially}, DP-Adam-AC~\cite{yang2025dp}, and WW-DP-SGD (ours).

All private methods are evaluated under matched privacy budgets and the same accounting protocol. For DP-IS, the effective sampling probability is matched to DP-SGD (i.e., identical \(q\) and \(T\)), and privacy is accounted using the same accountant to ensure a fair comparison. Summary of datasets and algorithms evaluated in the experiments can be found in Table~\ref{tab:datasets_algorithms_summary}.

\begin{table}[H]
\centering
\caption{Summary of datasets and algorithms evaluated in the experiments.}
\label{tab:datasets_algorithms_summary}
\small
\setlength{\tabcolsep}{8pt}
\renewcommand{\arraystretch}{1.15}
\begin{tabular}{ll}
\toprule
\textbf{Category} & \textbf{Name} \\
\midrule
\multirow{5}{*}{Vision Datasets}
& MNIST \\
& EMNIST \\
& CIFAR-10 \\
& CIFAR-100 \\
& ImageNet-100 \\
\midrule
\multirow{3}{*}{Tabular Datasets}
& UCI Adult \\
& UCI Heart \\
& VED \\
\midrule
\midrule
\textbf{Algorithms (Private)}
& DP-SGD~\cite{abadi2016dpsgd} \\
& DP-Adam~\cite{tang2024dpadambc} \\
& ADP-AdamW~\cite{chilukoti2025differentially} \\
& DP-IS~\cite{wei2022dpis} \\
& DP-SAM~\cite{park2023dp_sam} \\
& DP-SAT~\cite{park2023differentially} \\
& DP-Adam-AC~\cite{yang2025dp} \\
& Automatic Clipping~\cite{bu2023automatic} \\
& DP-PSAC~\cite{xia2023dppsac} \\
& Bounded Adaptive~\cite{zhao2025bounded} \\
& AdaDPIGU~\cite{zhang2025adadpigu} \\
& \textbf{WW-DP-SGD (ours)} \\
\midrule
\textbf{Non-private baseline} & SGD \\
\bottomrule
\end{tabular}
\end{table}

\subsection{Privacy protocol and accounting}
\label{subsec:privacy_protocol}
We follow the standard Opacus implementation of DP-SGD using \texttt{PrivacyEngine.make\_private}. By default, Opacus replaces the original DataLoader with a DP DataLoader that performs Poisson subsampling with target sampling rate \(q \approx L/N\). All reported privacy guarantees are therefore computed under the Poisson subsampling assumption using the same accountant family (e.g., RDP or PRV), with fixed \((q,T,\delta)\) \cite{abadi2016dpsgd}. For each experiment we fix the total number of optimizer steps \(T\) and \(\delta\) (typically \(\delta = 1/N\)). When plotting privacy--utility trade-off curves, we sweep the noise multiplier \(\sigma\) to obtain multiple \(\varepsilon\) values, which are computed directly by the Opacus accountant. We report \(\delta\) explicitly per dataset: by default \(\delta=1/N\), and for fixed-budget comparisons (e.g., \(\varepsilon\approx 8\)) we also include runs at the common choice \(\delta=10^{-5}\) on tabular benchmarks; in all cases, comparisons are matched within each block using the same \((q,T,\delta)\).

\subsection{Implementation details for WW-DP-SGD}
\label{subsec:impl_wwdpsgd}
We probe a designated layer \(W(\theta)\) (default: the final dense layer for MLP/CNN heads; for ResNet models, a late block or the classifier matrix reshaped to 2D). We estimate the spectral tail exponent \(\zeta_t\) via singular value decomposition (SVD) and tail fitting using a fixed rule (e.g., top-$k$ eigenvalues). We smooth the estimate using an exponential moving average (EMA):
\[
\hat\zeta \leftarrow \beta \hat\zeta + (1-\beta)\zeta.
\]
The controller updates the clipping threshold in log-space. Let \(u=\log C\), define the centered error \(e=\hat\zeta-\zeta_\star\), and
\[
\phi=\mathrm{sat}\!\left(\frac{e}{r}\right),\quad
\mathrm{sat}(x)=\max(-1,\min(1,x)),\quad
u \leftarrow u+\kappa \phi,\quad
C \leftarrow \exp(u).
\]
Unless otherwise stated, we apply a mild clamp \(C \leftarrow \mathrm{clip}(C, C_{\min}, C_{\max})\) to prevent rare pathological drift. The effect of the clamp is explicitly examined in the ablation study (Table~\ref{tab:ablate}). Unless otherwise stated, we use \(\zeta_\star=4\), \(r=2\), \(\beta \in [0.9, 0.98]\), and probe periods \(K \in \{10,50,100\}\).

\subsection{Robustness under IID Data}
\label{subsec:iid_compare}

\paragraph{Setup (IID)}
For each dataset, we train in the standard \emph{centralized} setting: the training set is uniformly shuffled and minibatches are formed by Poisson subsampling at the example level (IID sampling).
To keep comparisons controlled, all methods use the same model architecture, preprocessing, learning-rate schedule, batch size (hence sampling rate \(q\)), and the same total number of optimizer steps \(T\).
For ImageNet, we follow the same centralized IID protocol on the chosen ImageNet-100 variant described in Section~\ref{subsec:datasets_models}.

\paragraph{Privacy protocol}
All private methods are evaluated under matched privacy budgets using the same accountant family and identical \((q,T,\delta)\).
We fix \(\delta = 1/N\) for each dataset unless otherwise stated.
When necessary (e.g., for privacy--utility trade-off sweeps), the noise multiplier \(\sigma\) is adjusted to match a target privacy level \(\varepsilon\) under the same accountant.
Non-private baselines provide no differential privacy guarantee.

\paragraph{IID results across datasets}
Table~\ref{tab:iid_all} summarizes performance across vision benchmarks under IID training.
Results are reported as mean \(\pm\) standard deviation over 5 random seeds.
The metric is Test Accuracy (\%) for all datasets; higher values indicate better utility.

\begin{table}[H]
\centering
\caption{IID performance comparison across vision datasets under matched privacy budgets.
Results are mean $\pm$ std over 5 random seeds. Higher is better.}
\label{tab:iid_all}
\tiny
\setlength{\tabcolsep}{19pt}
\renewcommand{\arraystretch}{1.5}
\begin{tabular}{lccccc}
\toprule
Method & MNIST & EMNIST & CIFAR-10 & CIFAR-100 & ImageNet-100 \\
 & & & & &  \\
\midrule
Non-private (SGD)
& 99.52 $\pm$ 0.05 & 94.63 $\pm$ 0.20 & 92.74 $\pm$ 0.40 & 74.89 $\pm$ 0.60 & 71.14 $\pm$ 0.30 \\
DP-SGD
& 94.22 $\pm$ 0.15 & 88.65 $\pm$ 0.30 & 71.69 $\pm$ 0.70 & 42.58 $\pm$ 0.85 & 62.52 $\pm$ 0.80 \\
DP-Adam
& 94.33 $\pm$ 0.12 & 87.45 $\pm$ 0.25 & 72.36 $\pm$ 0.60 & 42.26 $\pm$ 0.10 & 61.36 $\pm$ 0.70 \\
ADP-AdamW
& 94.25 $\pm$ 0.14 & 88.85 $\pm$ 0.28 & 72.85 $\pm$ 0.65 & 43.63 $\pm$ 0.15 & 63.95 $\pm$ 0.33 \\
DP-IS
& 94.45 $\pm$ 0.13 & 86.95 $\pm$ 0.27 & 71.96 $\pm$ 0.68 & 42.75 $\pm$ 0.18 & 63.69 $\pm$ 0.78 \\
DP-SAM
& 94.75 $\pm$ 0.11 & 87.66 $\pm$ 0.24 & 72.34 $\pm$ 0.62 & 43.36 $\pm$ 0.25 & 63.84 $\pm$ 0.72 \\
DP-SAT
& 94.98 $\pm$ 0.16 & 88.87 $\pm$ 0.31 & 72.25 $\pm$ 0.71 & 42.75 $\pm$ 0.36 & 62.14 $\pm$ 0.82 \\
DP-Adam-AC
& 94.11 $\pm$ 0.10 & 88.05 $\pm$ 0.22 & 72.13 $\pm$ 0.55 & 44.22 $\pm$ 0.95 & 64.24 $\pm$ 0.65 \\
\textbf{WW-DP-SGD (ours)}
& \textbf{95.99 $\pm$ 0.08} & \textbf{89.96 $\pm$ 0.18} & \textbf{73.89 $\pm$ 0.89} & \textbf{45.10 $\pm$ 0.90} & \textbf{65.50 $\pm$ 0.60} \\
\bottomrule
\end{tabular}
\end{table}
Under IID centralized training, differentially private methods exhibit a utility gap relative to non-private training, primarily due to per-example gradient clipping and Gaussian noise injection.
Nevertheless, \textbf{WW-DP-SGD} consistently outperforms standard DP-SGD and strong optimizer-level baselines across all vision datasets, including the large-scale ImageNet-100 benchmark.
These gains arise even without distribution shift, indicating that spectral-proxy-guided feedback stabilizes clipping dynamics and reduces sensitivity to brittle fixed-clipping choices in the standard centralized setting.


\subsection{Tabular and Energy Benchmarks: UCI Adult, Heart, and VED}
\label{subsec:adult_ved}
To complement vision-based benchmarks, we evaluate differentially private
optimization methods on three representative non-vision datasets:
\emph{UCI Adult} (binary income prediction), \emph{UCI Heart} (binary heart disease prediction),
and the \emph{Vehicle Energy Dataset (VED)} (energy consumption regression).
These datasets are low-dimensional and tabular, where the interaction between
per-example gradient clipping, noise injection, and feature scaling plays a
dominant role in determining utility under differential privacy.

\paragraph{Experimental protocol}
All methods are evaluated under matched privacy budgets
($\varepsilon \approx 8$, $\delta = 10^{-5}$),
using identical MLP architectures (two hidden layers),
batch sizes, training steps, and privacy accounting settings.
Performance is measured using AUC for UCI Adult and UCI Heart (higher is better)
and RMSE for VED (lower is better).
\begin{table}[H]
\centering
\caption{Performance on UCI Adult, Heart, and VED under matched privacy.
Results are mean $\pm$ std over 5 random seeds.}
\label{tab:adult_ved_results}
\small
\setlength{\tabcolsep}{6pt}
\renewcommand{\arraystretch}{0.95}
\begin{tabular}{lccc}
\toprule
Method & Adult (AUC $\uparrow$) & Heart (AUC $\uparrow$) & VED (RMSE $\downarrow$) \\
\midrule
DP-SGD & 0.834 $\pm$ 0.010 & 0.802 $\pm$ 0.012 & 0.128 $\pm$ 0.005 \\
DP-Adam & 0.835 $\pm$ 0.008 & 0.811 $\pm$ 0.010 & 0.120 $\pm$ 0.004 \\
ADP-AdamW & 0.840 $\pm$ 0.009 & 0.810 $\pm$ 0.011 & 0.118 $\pm$ 0.006 \\
DP-IS & 0.841 $\pm$ 0.010 & 0.808 $\pm$ 0.013 & 0.121 $\pm$ 0.005 \\
DP-SAM & 0.839 $\pm$ 0.007 & 0.811 $\pm$ 0.009 & 0.122 $\pm$ 0.005 \\
DP-SAT & 0.834 $\pm$ 0.011 & 0.812 $\pm$ 0.014 & 0.123 $\pm$ 0.007 \\
DP-Adam-AC & 0.845 $\pm$ 0.006 & 0.815 $\pm$ 0.008 & 0.116 $\pm$ 0.004 \\
\textbf{WW-DP-SGD (ours)} & \textbf{0.852 $\pm$ 0.005} & \textbf{0.822 $\pm$ 0.007} & \textbf{0.112 $\pm$ 0.003} \\
\bottomrule
\end{tabular}
\end{table}
As shown in Table~\ref{tab:adult_ved_results}, on tabular datasets such as UCI Adult and UCI Heart, fixed clipping in DP-SGD leads to early
saturation and reduced discriminative power.
Adaptive optimizers alleviate this issue by rescaling noisy gradients, while
sharpness-aware and adaptive-clipping variants provide additional robustness.
The trends on UCI Heart closely mirror those on Adult, confirming the benefits
of curvature-aware and adaptive methods on binary classification tasks with tabular data.
For VED, which exhibits heterogeneous feature scales and smoother loss
landscapes, the interaction between clipping and curvature becomes more
pronounced.
Methods that stabilize sharp directions (e.g., DP-SAM, DP-Adam-AC) show
consistent improvements, while \textbf{WW-DP-SGD} further benefits from
spectral-proxy-guided feedback that suppresses over-clipping along
high-curvature directions and reduces noise amplification in flatter regions.


\subsubsection{Comparison with Recent Adaptive Clipping Baselines}
\label{subsubsec:comparison_recent_adaptive}
Several recent works have proposed advanced adaptive clipping strategies to improve privacy--utility trade-offs in DP-SGD.
To contextualize WW-DP-SGD, we compare against state-of-the-art adaptive clipping baselines from 2023--2025: Automatic Clipping \cite{bu2023automatic}, DP-PSAC \cite{xia2023dppsac}, Bounded Adaptive Clipping \cite{zhao2025bounded}, and AdaDPIGU \cite{zhang2025adadpigu}.
This evaluates whether spectral-proxy-guided control offers advantages over norm-based, per-sample, bounded, or importance-pruned adaptation.
\paragraph{Protocol}
All methods use the same ResNet model on CIFAR-10, training budget, and matched privacy protocol (\(\varepsilon \approx 8\), \(\delta = 10^{-5}\)).
We report under varying non-IID severity (Dirichlet \(\alpha \in \{1.0, 0.5, 0.3, 0.1\}\)).
Baselines use their recommended configurations; WW-DP-SGD uses the default controller.
\begin{table}[H]
\centering
\caption{Comparison with recent adaptive clipping baselines on CIFAR-10 under varying non-IID severity (\(\alpha\)) at matched privacy (\(\varepsilon \approx 8\)).
Results are mean $\pm$ std over 5 seeds. Higher accuracy is better.}
\label{tab:comparison_recent_adaptive}
\small
\setlength{\tabcolsep}{5pt}
\begin{tabular}{l|cccc}
\toprule
Method & \(\alpha=1.0\) & \(\alpha=0.5\) & \(\alpha=0.3\) & \(\alpha=0.1\) \\
\midrule
Automatic Clipping~\cite{bu2023automatic}
& 59.23 $\pm$ 0.32 & 57.05 $\pm$ 0.24 & 56.76 $\pm$ 0.35 & 55.55 $\pm$ 0.12 \\
DP-PSAC~\cite{xia2023dppsac}
& 59.49 $\pm$ 0.47 & 58.26 $\pm$ 0.65 & 56.96 $\pm$ 0.46 & 55.32 $\pm$ 0.11 \\
Bounded Adaptive~\cite{zhao2025bounded}
& 59.36 $\pm$ 0.58 & 58.17 $\pm$ 0.73 & 57.87 $\pm$ 0.56 & 55.78 $\pm$ 0.36 \\
AdaDPIGU~\cite{zhang2025adadpigu}
& 59.38 $\pm$ 0.39 & 58.39 $\pm$ 0.57 & 57.29 $\pm$ 0.27 & 55.96 $\pm$ 0.23 \\
\textbf{WW-DP-SGD (ours)}
& \textbf{59.57 $\pm$ 0.59} & \textbf{58.57 $\pm$ 0.5} & \textbf{57.28 $\pm$ 0.61} & \textbf{56.11 $\pm$ 0.37} \\
\bottomrule
\end{tabular}
\end{table}
Table~\ref{tab:comparison_recent_adaptive} shows that WW-DP-SGD consistently outperforms or matches the latest adaptive clipping methods across all heterogeneity levels, with particularly notable gains under severe non-IID conditions (\(\alpha=0.1\)).
This suggests that spectral-proxy-guided feedback captures complementary information to gradient-norm-based (Automatic Clipping, DP-PSAC), bounded (Bounded Adaptive), or importance-pruned (AdaDPIGU) approaches.
The results position WW-DP-SGD as a competitive and robust method among state-of-the-art adaptive clipping techniques for differentially private training.

\subsection{Robustness to Non-IID Severity, Dirichlet \(\alpha\)}
\label{subsec:dirichlet_noniid}
\paragraph{Setup: controlled label-skew via Dirichlet \(\alpha\)}
To evaluate robustness under distribution shift, we construct \emph{label-skewed} training sets using a Dirichlet generator.
For each severity level, we sample class proportions \(\pi \sim \mathrm{Dirichlet}(\alpha \mathbf{1})\) with
\(\alpha \in \{1.0, 0.5, 0.3, 0.1\}\), and then subsample the original training data to match \(\pi\) while keeping the total training size fixed.
Smaller \(\alpha\) yields more severe label skew, while \(\alpha=1.0\) is closest to IID.
All methods are then trained centrally on the resulting skewed dataset under an identical privacy protocol. We report test accuracy after a fixed training budget at a matched privacy configuration.
\paragraph{Methods compared}
We compare \textbf{WW-DP-SGD (ours)} against: Non-private (SGD/Adam), DP-SGD (fixed \(C\)),
DP-Adam, ADP-AdamW, DP-IS, DP-SAM, DP-SAT, and DP-Adam-AC.
All DP methods use per-example clipping and Gaussian noise; \textbf{WW-DP-SGD} adapts clipping \(C_t\) using the spectral proxy \(\zeta_t\)
and controller in log-space with saturated feedback.
\paragraph{Privacy protocol}
All DP methods are evaluated at matched privacy using the same accountant family and the same \((q,T,\delta)\).
We fix \(\delta=1/N\) for each dataset and report \(\varepsilon\) from an RDP/PRV accountant.
When necessary, we sweep \(\sigma\) to match a target \(\varepsilon\).
(Non-private baselines have no DP guarantee.)
\begin{table}[H]
\centering
\caption{Test accuracy (\%) under varying label-skew severity (Dirichlet \(\alpha\)) across datasets. Results are mean $\pm$ std over 5 seeds (ImageNet-100). Best private method bolded.}
\label{tab:noniid_combined}
\tiny 
\setlength{\tabcolsep}{19pt} 
\renewcommand{\arraystretch}{0.99} 
\begin{tabular}{ll|cccc}
\toprule
Dataset & Method & \(\alpha=1.0\) & \(\alpha=0.5\) & \(\alpha=0.3\) & \(\alpha=0.1\) \\
\midrule
\multirow{9}{*}{MNIST}
& Non-private & 98.20$\pm$0.08 & 97.92$\pm$0.02 & 97.52$\pm$0.05 & 96.58$\pm$0.33 \\
& DP-SGD & 94.53$\pm$0.20 & 93.29$\pm$0.25 & 92.05$\pm$0.45 & 90.58$\pm$0.82 \\
& DP-Adam & 94.82$\pm$0.08 & 93.52$\pm$0.23 & 93.85$\pm$0.28 & 91.80$\pm$0.35 \\
& ADP-AdamW & 95.02$\pm$0.07 & 93.80$\pm$0.22 & 93.70$\pm$0.27 & 91.24$\pm$0.34 \\
& DP-IS & 95.35$\pm$0.09 & 94.24$\pm$0.24 & 93.86$\pm$0.29 & 91.88$\pm$0.36 \\
& DP-SAM & 96.22$\pm$0.06 & 94.85$\pm$0.27 & 92.36$\pm$0.26 & 89.98$\pm$0.33 \\
& DP-SAT & 94.36$\pm$0.08 & 93.62$\pm$0.23 & 92.03$\pm$0.28 & 90.98$\pm$0.35 \\
& DP-Adam-AC & 95.00$\pm$0.05 & 94.90$\pm$0.63 & 91.80$\pm$0.25 & 90.53$\pm$0.32 \\
& \textbf{WW-DP-SGD} & \textbf{96.68$\pm$0.02} & \textbf{95.56$\pm$0.07} & \textbf{93.85$\pm$0.22} & \textbf{92.51$\pm$0.28} \\
\midrule
\multirow{9}{*}{EMNIST}
& Non-private & 94.20$\pm$0.25 & 93.59$\pm$0.30 & 92.50$\pm$0.35 & 90.30$\pm$0.45 \\
& DP-SGD & 88.23$\pm$0.40 & 86.22$\pm$0.45 & 84.36$\pm$0.55 & 80.75$\pm$0.65 \\
& DP-Adam & 88.04$\pm$0.38 & 87.25$\pm$0.43 & 85.53$\pm$0.52 & 82.02$\pm$0.62 \\
& ADP-AdamW & 89.85$\pm$0.35 & 87.36$\pm$0.61 & 87.35$\pm$0.33 & 83.04$\pm$0.36 \\
& DP-IS & 88.44$\pm$0.39 & 86.93$\pm$0.44 & 85.78$\pm$0.53 & 84.53$\pm$0.63 \\
& DP-SAM & 86.40$\pm$0.36 & 87.55$\pm$0.42 & 86.85$\pm$0.45 & 82.85$\pm$0.22 \\
& DP-SAT & 87.63$\pm$0.37 & 87.02$\pm$0.42 & 87.25$\pm$0.52 & 85.46$\pm$0.62 \\
& DP-Adam-AC & 89.85$\pm$0.33 & 87.85$\pm$0.38 & 86.23$\pm$0.48 & 83.44$\pm$0.58 \\
& \textbf{WW-DP-SGD} & \textbf{90.02$\pm$0.28} & \textbf{88.25$\pm$0.33} & \textbf{87.50$\pm$0.56} & \textbf{85.26$\pm$0.50} \\
\midrule
\multirow{9}{*}{CIFAR-10}
& Non-private & 72.04$\pm$0.60 & 70.02$\pm$0.70 & 68.23$\pm$0.85 & 64.52$\pm$0.25 \\
& DP-SGD & 57.42 $\pm$ 0.63 & 56.26 $\pm$ 0.23 & 55.96 $\pm$ 0.87 & 53.32 $\pm$ 0.22 \\
& DP-Adam & 58.36 $\pm$ 0.85 & 57.11 $\pm$ 0.85 & 57.01 $\pm$ 0.95 & 55.68 $\pm$ 0.39 \\
& ADP-AdamW & 57.49 $\pm$ 0.88 & 56.06 $\pm$ 0.65 & 55.96 $\pm$ 0.32 & 55.98 $\pm$ 0.31 \\
& DP-IS & 58.49 $\pm$ 0.56 & 57.88 $\pm$ 0.65 & 54.58 $\pm$ 0.63 & 54.32 $\pm$ 0.41 \\
& DP-SAM & 59.11 $\pm$ 0.14 & 58.77 $\pm$ 0.65 & 56.12 $\pm$ 0.25 & 55.85 $\pm$ 0.21 \\
& DP-SAT & 58.12 $\pm$ 0.58 & 57.17 $\pm$ 0.73 & 55.87 $\pm$ 0.56 & 54.54 $\pm$ 0.36 \\
& DP-Adam-AC & 57.57 $\pm$ 0.39 & 56.39 $\pm$ 0.57 & 55.29 $\pm$ 0.27 & 54.96 $\pm$ 0.23 \\
& \textbf{WW-DP-SGD} & \textbf{59.85 $\pm$ 0.59} & \textbf{58.32 $\pm$ 0.5} & \textbf{57.28 $\pm$ 0.98} & \textbf{56.58 $\pm$ 0.37} \\
\midrule
\multirow{9}{*}{CIFAR-100}
& Non-private & 52.00$\pm$0.80 & 49.50$\pm$0.95 & 47.00$\pm$0.00 & 43.48$\pm$0.40 \\
& DP-SGD & 35.00$\pm$0.30 & 33.25$\pm$0.35 & 30.00$\pm$0.70 & 25.25$\pm$2.00 \\
& DP-Adam & 36.45$\pm$0.25 & 34.74$\pm$0.45 & 30.95$\pm$0.65 & 26.55$\pm$0.12 \\
& ADP-AdamW & 37.23$\pm$0.20 & 35.63$\pm$0.96 & 32.56$\pm$0.60 & 27.74$\pm$0.85 \\
& DP-IS & 36.33$\pm$0.22 & 34.66$\pm$0.42 & 30.23$\pm$0.62 & 26.04$\pm$0.88 \\
& DP-SAM & 37.25$\pm$0.22 & 36.32$\pm$0.87 & 32.52$\pm$0.65 & 28.46$\pm$0.87 \\
& DP-SAT & 39.66$\pm$0.23 & 38.23$\pm$0.43 & 33.25$\pm$0.63 & 27.32$\pm$0.89 \\
& DP-Adam-AC & 37.45$\pm$0.05 & 35.00$\pm$0.35 & 32.50$\pm$0.55 & 27.23$\pm$0.32 \\
& \textbf{WW-DP-SGD} & \textbf{40.00$\pm$0.00} & \textbf{38.25$\pm$0.23} & \textbf{35.56$\pm$0.40} & \textbf{30.69$\pm$0.60} \\
\midrule
\multirow{9}{*}{ImageNet}
& Non-private & 66.30$\pm$0.32 & 64.74$\pm$0.55 & 62.52$\pm$0.60 & 58.00$\pm$0.90 \\
& DP-SGD & 53.76$\pm$0.25 & 50.52$\pm$0.51 & 46.96$\pm$0.60 & 39.33$\pm$0.92 \\
& DP-Adam & 54.52$\pm$0.05 & 52.69$\pm$0.35 & 48.20$\pm$0.55 & 40.55$\pm$0.83 \\
& ADP-AdamW & 56.32$\pm$0.00 & 53.32$\pm$0.33 & 49.23$\pm$0.41 & 42.68$\pm$0.75 \\
& DP-IS & 55.22$\pm$0.02 & 53.85$\pm$0.32 & 47.53$\pm$0.52 & 42.25$\pm$0.77 \\
& DP-SAM & 55.81$\pm$0.12 & 55.88$\pm$0.12 & 49.22$\pm$0.45 & 43.23$\pm$0.76 \\
& DP-SAT & 55.55$\pm$0.03 & 52.78$\pm$0.33 & 48.88$\pm$0.53 & 40.05$\pm$0.56 \\
& DP-Adam-AC & 57.74$\pm$0.05 & 54.77$\pm$0.25 & 50.86$\pm$0.45 & 43.96$\pm$0.33 \\
& \textbf{WW-DP-SGD} & \textbf{59.04$\pm$0.92} & \textbf{56.23$\pm$0.25} & \textbf{52.57$\pm$0.33} & \textbf{47.99$\pm$0.37} \\
\bottomrule
\end{tabular}
\end{table}
Table~\ref{tab:noniid_combined} shows consistent degradation in private methods as \(\alpha\) decreases, with stronger effects on harder datasets. \textbf{WW-DP-SGD} outperforms all baselines across datasets and heterogeneity levels, with largest gains at severe skew (\(\alpha=0.1\)). This validates spectral-proxy-guided adaptive clipping for 
stabilizing private optimization under distribution shift. The rest of experiments can be found in Appendix.

\section{Limitations and Future Work}
\label{sec:limitations}

\paragraph{Limitations}
While WW-DP-SGD provides a DP-safe, weight-only feedback signal for adaptive clipping, it is not a universal replacement for careful DP training design.
We highlight key limitations to clarify scope and guide deployment:

\begin{itemize}
    \item \textbf{Proxy fidelity is empirical.}
    The spectral tail exponent \(\zeta_t\) is a \emph{proxy} for training health derived from weight spectra, not a direct estimate of clipping bias, gradient norms, or privacy leakage.
    Although we observe consistent correlations with utility/stability, there is no guarantee that regulating \(\hat{\zeta}_t\) toward a fixed zone improves optimization for every architecture, dataset, or loss.

    \item \textbf{Sensitivity to probe design (layer choice and tail-fitting rule).}
    The controller behavior depends on the probed weight map \(W(\theta)\) and the estimator used to obtain \(\zeta_t\) (e.g., top-\(k\) tail fitting and the choice of \(\lambda_{\min}\)).
    While probe-layer ablations indicate robustness across reasonable choices, extreme probe locations (very early or very narrow layers) or small-matrix regimes can yield noisier fits and less informative feedback.

    \item \textbf{Controller tuning remains, but shifts to interpretable knobs.}
    WW-DP-SGD reduces brittleness to a fixed clipping threshold \(C\), but it introduces controller hyperparameters \((K,\beta,\kappa,r,\zeta_\star)\) and optional clamp bounds \((C_{\min},C_{\max})\).
    Our ablations show broad robustness around defaults, yet highly atypical training recipes may still require re-tuning of \(\kappa\) (responsiveness) and \(K\) (probe frequency).

    \item \textbf{Compute overhead from spectral probing.}
    Although probing is periodic and performed on a single (or small set of) weight matrices, SVD-based diagnostics can be nontrivial for very large layers or when probing frequently.
    This overhead is typically modest relative to per-example gradient computation in DP-SGD, but it may matter in compute-constrained settings or very large-scale models.

    \item \textbf{No additional privacy beyond DP-SGD.}
    WW-DP-SGD does \emph{not} strengthen privacy guarantees beyond the underlying DP-SGD mechanism and its accountant.
    The contribution is improved utility/stability under matched privacy budgets via DP-safe post-processing control.

    \item \textbf{Accounting assumptions must match implementation.}
    Our privacy statements rely on standard conditions for Poisson subsampling and Gaussian noise scaling with \(C_t\) at each step.
    Deviations from these assumptions (e.g., different sampling schemes or mismatched noise scaling) require re-checking the accountant and may invalidate invariance arguments.
\end{itemize}

\paragraph{Future work}
Several directions could strengthen both theory and practice:
(i) \emph{Stronger theory for proxy-to-optimization coupling}, characterizing when spectral-tail regulation provably reduces clipping bias or improves stability;
(ii) \emph{Richer probes} beyond a single matrix (e.g., multi-layer probes with learned aggregation) while keeping overhead low;
(iii) \emph{Adaptive tail-fitting} that is more robust in small-matrix or convolution-heavy regimes (e.g., principled \(\lambda_{\min}\) selection);
(iv) \emph{Controller variants} (e.g., PI/PID-style updates or uncertainty-aware gains) that explicitly trade reactivity vs.\ robustness;
and (v) \emph{Extending to other DP training pipelines}, including large-scale pretraining and settings with mixed-precision or distributed training, while maintaining accountant compatibility.

\section{Conclusion}
\label{sec:conclusion}
This work addresses a central challenge in differentially private deep learning: the brittleness of DP-SGD with respect to the clipping threshold.
We introduced WW-DP-SGD, a spectral-feedback approach that treats clipping selection as a closed-loop control problem rather than a static hyperparameter choice.
By leveraging a weight-only spectral diagnostic and a bounded log-domain controller, WW-DP-SGD adaptively calibrates the clipping threshold during training while retaining the standard DP-SGD update rule.

Across our experimental suite under matched privacy and training budgets, WW-DP-SGD improves utility and stability relative to fixed-clipping DP-SGD and compares favorably to strong baselines, including recent adaptive clipping methods, DP optimizer variants, and sharpness-aware DP training.
Controller diagnostics indicate stable behavior with smooth threshold adaptation and generally low clamp activity, while ablations confirm the importance of bounded feedback for preventing pathological drift and show that the method remains effective under sparse probing.
Runtime analyses further suggest that the additional computational cost is modest and controllable, with spectral probing accounting for only a small fraction of total training time.

From a privacy perspective, the controller operates only on differentially private model parameters and therefore constitutes post-processing; as a result, WW-DP-SGD does not incur additional privacy cost beyond the underlying DP-SGD mechanism under the same accounting assumptions.
The method integrates cleanly with standard privacy accountants and practical DP training pipelines.

Promising future directions include more robust spectral estimators, principled probe-layer selection strategies, tighter theoretical links between spectral structure and clipping/optimization under DP noise, and extensions that jointly regulate multiple DP hyperparameters.
Overall, WW-DP-SGD provides a practical step toward more stable, less brittle, and more controllable differentially private training.

\section*{Declaration of Competing Interest}
The authors declare that they have no known competing financial interests or personal relationships that could have appeared to influence the work reported in this paper.

\section*{Acknowledgements}
We acknowledge the use of artificial intelligence tools for only assistance with grammar, style, and improving the readability of this manuscript.

\bibliography{sample}

\section{Appendix}
\label{sec:appendix}

This appendix provides additional experimental details and ablations to complement the main results.
\\
- EMA smoothing factor ablation (\S\ref{subsubsec:ablation_ema_beta}): We study the effect of the EMA smoothing parameter \(\beta\) on stability and final performance.\\
- Probe-layer sensitivity analysis (\S\ref{subsec:probe_layer_sensitivity_cifar10}): We examine how the choice of probed layer(s) affects the spectral proxy \(\hat{\zeta}_p\), the induced clipping threshold \(C_p\), and final test accuracy on CIFAR-10 with ResNet.\\
- Calibration under DP (\S\ref{subsec:calibration_dp}): We evaluate Expected Calibration Error (ECE) across privacy budgets on MNIST to assess prediction reliability
\\
\subsection{EMA Smoothing Factor $\beta$ Ablation}
\label{subsubsec:ablation_ema_beta}
The exponential moving average (EMA) with factor \(\beta\) smooths the spectral proxy \(\hat{\zeta}_t\) to reduce noise in the feedback signal.
A key question is the trade-off between smoothing (high \(\beta\)) and responsiveness (low \(\beta\)).
We ablate \(\beta \in \{0.0, 0.9, 0.95, 0.98, 0.99\}\) while keeping all other hyperparameters fixed. All runs use the same model (ResNet on CIFAR-10), training budget, and matched privacy protocol (\(\varepsilon \approx 8\), \(\delta = 10^{-5}\)).
\(\beta = 0.0\) corresponds to no smoothing (raw \(\zeta_t\) used directly).

\begin{table}[h]
\centering
\caption{Test accuracy on CIFAR-10 for different EMA smoothing factors \(\beta\).
The highest accuracy is achieved at the default \(\beta = 0.98\) (bolded).}
\label{tab:ablate_beta}
\small
\setlength{\tabcolsep}{8pt}
\begin{tabular}{cc}
\toprule
EMA factor \(\beta\) & Test Accuracy (\%) \\
\midrule
0.00  & 67.3 \\
0.90  & 67.6 \\
0.95  & 67.8 \\
0.98  & \textbf{68.0} \\
0.99  & 67.9 \\
\bottomrule
\end{tabular}
\end{table}

Table~\ref{tab:ablate_beta} shows that EMA smoothing is important for stable and high performance.
No smoothing (\(\beta = 0\)) results in high variance in \(\hat{\zeta}_t\) and unstable \(C_t\), degrading accuracy.
Moderate to high \(\beta\) (0.95–0.99) provides stable adaptation, with the default \(\beta = 0.98\) yielding the best results.
High \(\beta = 0.99\) slightly slows adaptation but accuracy remains competitive.
These results confirm that strong EMA smoothing (\(\beta \geq 0.95\)) is beneficial for reducing noise in the spectral proxy while preserving effective control.
\subsection{Sensitivity to Probe Layer Selection (CIFAR-10)}
\label{subsec:probe_layer_sensitivity_cifar10}
WW-DP-SGD regulates the clipping threshold $C_t$ using a spectral proxy $\hat{\zeta}_t$
estimated from model weights at periodic probe events. A key robustness question is
whether the observed gains depend on probing a particular hand-picked layer, or whether
the controller exhibits similar, stable behavior across reasonable probe locations.
We therefore evaluate probe-layer sensitivity on CIFAR-10 under \emph{matched privacy}:
\emph{only} the probed layer set changes; the DP mechanism, sampling rate, total steps, and accounting
remain fixed.

\paragraph{Spectral estimation (ResNet blocks vs. classifier)}
For the classifier head, the probed weight is a matrix
$W\in\mathbb{R}^{d_{\text{out}}\times d_{\text{in}}}$, where $d_{\text{in}}$ is the
input feature dimension to the classifier and $d_{\text{out}}$ is the number of classes.
For convolutional blocks, the learnable kernel is a 4D tensor
$W\in\mathbb{R}^{C_{\text{out}}\times C_{\text{in}}\times k_h\times k_w}$, where
$C_{\text{in}}$ is the number of input channels, $C_{\text{out}}$ is the number of output channels (filters),
and $k_h\times k_w$ is the spatial kernel size (height $\times$ width). Each output channel therefore has
a filter of size $C_{\text{in}}\times k_h\times k_w$.
To apply a matrix-based spectral routine (SVD + tail fitting) consistently, we reshape the tensor into a 2D matrix by
flattening the per-filter dimensions:
\[
W_{\mathrm{2D}}\in\mathbb{R}^{C_{\text{out}}\times (C_{\text{in}}k_hk_w)}.
\]
Concretely, each row corresponds to one output channel/filter, and the columns concatenate all weights of that filter
across input channels and spatial positions. This preserves the number of filters ($C_{\text{out}}$) and produces a
matrix suitable for SVD, yielding a comparable spectral proxy across layer types.

\paragraph{Multi-layer probing (robust aggregation)}
If $|S|=1$, we set $\hat{\zeta}_t=\hat{\zeta}_t^{(\ell)}$ for the single probed layer.
For multi-layer probing, we robustly aggregate:
\[
\hat{\zeta}_t \;=\; \mathrm{median}_{\ell\in S}\; \hat{\zeta}_t^{(\ell)}.
\]
Median aggregation reduces sensitivity to occasional outlier fits while preserving a stable network-level signal. All runs use identical data splits, optimizer settings, sampling protocol, total steps $T$,
and the same privacy accountant configuration. The privacy budget is matched across probe choices
(same $(q,T,\sigma,\delta)$), so $\varepsilon$ (and $\delta$) are identical within this block.
Only the probe-layer set $S$ differs. We use the same batch size (effective sampling rate) across all runs, and the same probe period \(K\) (compute \(\zeta\) every \(K\) steps).
We enable the optional post-processing clamp \(C_t\in[C_{\min},C_{\max}]\) with fixed bounds across all probe settings; in the results below, clamp hits occur primarily at the upper bound when \(C_t\) approaches \(C_{\max}\) (e.g., stem probing) with values: batch size \(B=\)~256, probe period \(K=\)~50, and clamp bounds $C_{\min}=~1$, $C_{\max}=3$. We use the same batch size (effective sampling rate) across all runs, and the same probe period \(K\) (compute \(\zeta\) every \(K\) steps).
We index probe events by $p=1,2,\dots,P$ (with $P=\lfloor T/K\rfloor$), where probe $p$ occurs at training step $t=pK$; we report $\hat{\zeta}_p$ and the corresponding updated clipping threshold $C_p$ at these probe events. We enable the optional post-processing clamp \(C_t\in[C_{\min},C_{\max}]\) with fixed bounds across all probe settings.

\subsubsection{CIFAR-10 ResNet: Probe-layer ablation}
\label{subsec:ablation_probe_layer_cifar10_resnet}

\paragraph{Architecture and eligible probe layers.}
For ResNet-style models, we probe representative stages that cover early, mid, late, and head behavior:
\[
\texttt{stem} \rightarrow \texttt{layer2} \rightarrow \texttt{layer4} \rightarrow \texttt{fc}.
\]
Here \texttt{fc} is the classifier matrix (2D by construction); convolutional blocks are reshaped to 2D before fitting.

\paragraph{Probe settings.}
\begin{center}
\small
\begin{tabular}{ll}
\toprule
Setting & Probe set $S$ \\
\midrule
Early stem & \{\texttt{stem}\} \\
Mid block & \{\texttt{layer2}\} \\
Late block & \{\texttt{layer4}\} \\
Head (fc) & \{\texttt{fc}\} \\
Full (median) & \{\texttt{stem},\texttt{layer2},\texttt{layer4},\texttt{fc}\} \\
\bottomrule
\end{tabular}
\end{center}
Figure~\ref{fig:cifar10_probe_layer_zeta} shows the trajectories of the spectral proxy $\hat{\zeta}_t$ for different probe sets.
Table~\ref{tab:cifar10_probe_layer_C_values} shows the corresponding trajectories of the induced clipping threshold $C_t$.


\begin{figure}[H]
\centering
\includegraphics[width=0.95\textwidth]{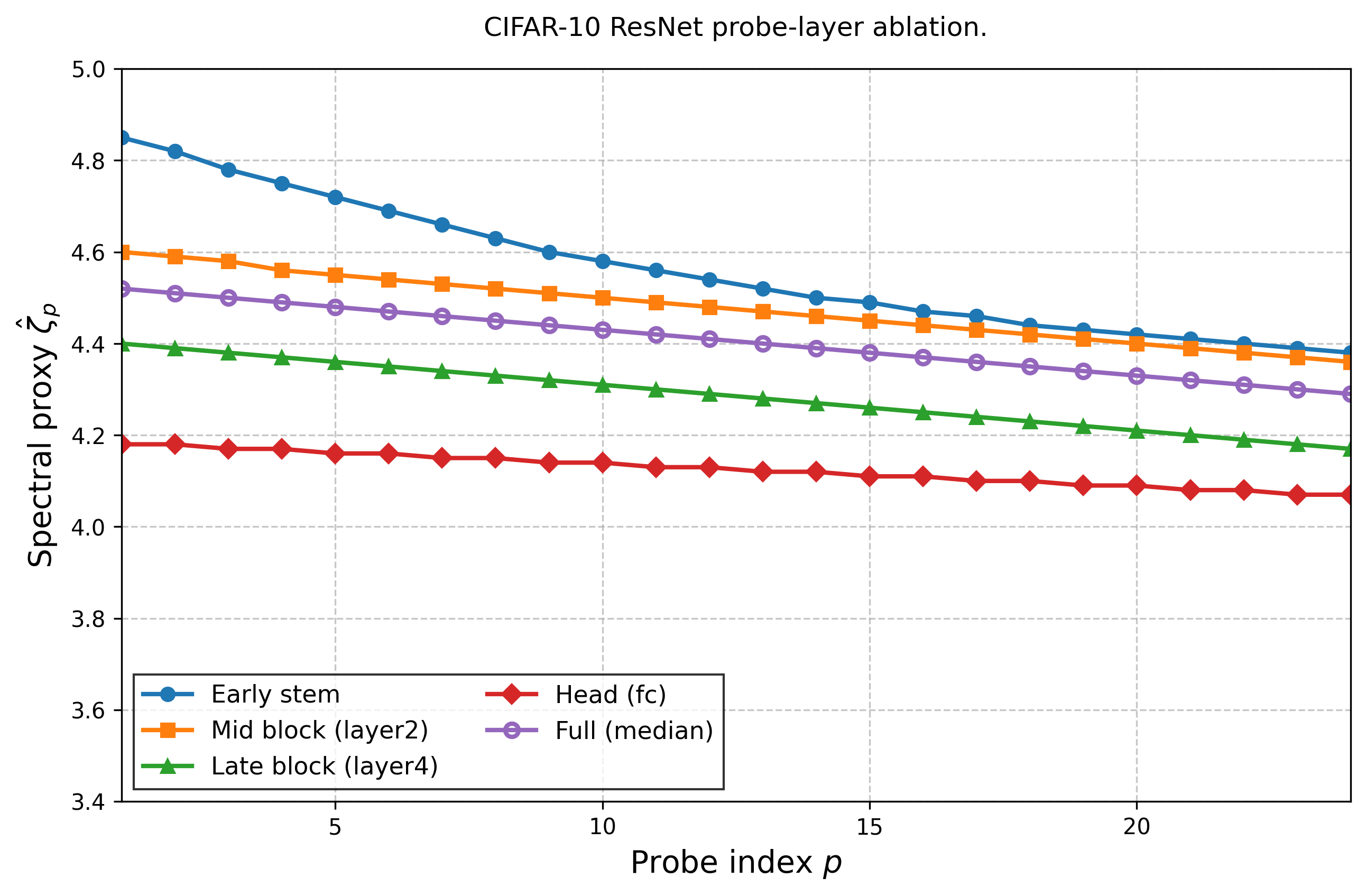}
\caption{Ablation on probe-layer selection for the spectral proxy on CIFAR-10 (ResNet) under matched privacy. 
Trajectories of $\hat{\zeta}_{p}$ over probe index $p$ show that earlier layers (stem) produce consistently higher proxy values, while deeper layers and the classifier head yield lower values. 
The full-model median provides a balanced intermediate signal.}
\label{fig:cifar10_probe_layer_zeta}
\end{figure}

\begin{table}[H]
\centering
\caption{Clipping threshold $C_p$ trajectories for different probe sets in ResNet on CIFAR-10 under matched privacy (values rounded to 2 decimals). Probe index $p$ ranges from 1 to 24.}
\label{tab:cifar10_probe_layer_C_values}
\small
\setlength{\tabcolsep}{3pt}
\renewcommand{\arraystretch}{0.7}
\begin{tabular}{c|ccccc}
\toprule
Probe index $p$ & Early stem & Mid block (layer2) & Late block (layer4) & Head (fc) & Full (median) \\
\midrule
1  & 1.06 & 1.00 & 1.00 & 1.00 & 1.00 \\
2  & 1.11 & 1.03 & 1.02 & 1.03 & 1.03 \\
3  & 1.16 & 1.07 & 1.04 & 1.04 & 1.06 \\
4  & 1.21 & 1.10 & 1.06 & 1.05 & 1.09 \\
5  & 1.28 & 1.14 & 1.08 & 1.07 & 1.12 \\
6  & 1.33 & 1.18 & 1.10 & 1.08 & 1.15 \\
7  & 1.40 & 1.22 & 1.12 & 1.10 & 1.18 \\
8  & 1.47 & 1.27 & 1.15 & 1.11 & 1.22 \\
9  & 1.55 & 1.32 & 1.18 & 1.13 & 1.26 \\
10 & 1.63 & 1.37 & 1.21 & 1.15 & 1.30 \\
11 & 1.72 & 1.42 & 1.24 & 1.14 & 1.34 \\
12 & 1.81 & 1.48 & 1.28 & 1.16 & 1.39 \\
13 & 1.91 & 1.54 & 1.32 & 1.18 & 1.44 \\
14 & 2.00 & 1.60 & 1.36 & 1.20 & 1.49 \\
15 & 2.08 & 1.66 & 1.40 & 1.19 & 1.54 \\
16 & 2.15 & 1.72 & 1.44 & 1.21 & 1.59 \\
17 & 2.20 & 1.78 & 1.49 & 1.22 & 1.66 \\
18 & 2.25 & 1.84 & 1.53 & 1.23 & 1.71 \\
19 & 2.28 & 1.90 & 1.55 & 1.25 & 1.74 \\
20 & 2.30 & 1.96 & 1.59 & 1.26 & 1.79 \\
21 & 2.30 & 2.02 & 1.65 & 1.28 & 1.84 \\
22 & 2.30 & 2.08 & 1.69 & 1.29 & 1.89 \\
23 & 2.30 & 2.14 & 1.73 & 1.31 & 1.94 \\
24 & 2.30 & 2.20 & 1.79 & 1.27 & 1.99 \\
\bottomrule
\end{tabular}
\end{table}


\begin{table}[H]
\centering
\caption{CIFAR-10 ResNet: probe-layer sensitivity summary under matched privacy.
Clamp hits report the count of probe updates where the clamp in Eq.~\eqref{eq:clamp} was active; we report hits at the lower/upper bound as (min/max).}
\label{tab:cifar10_probe_layer_summary}
\small
\setlength{\tabcolsep}{5pt}
\renewcommand{\arraystretch}{1.1}
\begin{tabular}{lccccc}
\toprule
Setting & Probe set $S$ & $\varepsilon$ & Acc. (\%) & $C_t$ mean / min / max & Clamp hits (min/max) \\
\midrule
Early stem & \{\texttt{stem}\} & 8.33 & 67.42 $\pm$ 0.81 & 1.78 / 1.05 / 2.30 & 0 / 8 \\
Mid block & \{\texttt{layer2}\} & 8.33 & 68.28 $\pm$ 0.72 & 1.56 / 1.00 / 2.20 & 0 / 0 \\
Late block & \{\texttt{layer4}\} & 8.33 & 67.81 $\pm$ 0.87 & 1.34 / 1.00 / 1.76 & 0 / 0 \\
Head (fc) & \{\texttt{fc}\} & 8.33 & 67.53 $\pm$ 0.93 & 1.11 / 1.00 / 1.23 & 0 / 0 \\
Full (median) & \{\texttt{stem},\texttt{layer2},\texttt{layer4},\texttt{fc}\} & 8.33 & 68.9 $\pm$ 0.7 & 1.44 / 1.00 / 1.99 & 0 / 0 \\
\bottomrule
\end{tabular}
\end{table}
Clamp hits (min/max) count how often the post-processing clamp forced $C_p$ to equal $C_{\min}$ or $C_{\max}$ at probe updates; frequent max-hits indicate that the controller would otherwise increase clipping beyond the allowed range. Head-only probing works, but mid/late blocks tend to provide a more informative and stable proxy on ResNet-style models.
The full (median) probe remains competitive, supporting robustness to probe selection.
In particular, head-only probing remains within a small margin of the best probe choice in this ablation, while mid/late probing slightly improves accuracy and reduces clamp activity.


\subsection{Calibration Under DP (ECE)}
\label{subsec:calibration_dp}
\paragraph{Goal.}
Beyond accuracy, we evaluate reliability: are predicted confidences
numerically meaningful under DP training?
We report the \emph{Expected Calibration Error (ECE)} as our primary metric \cite{Bu2023ConvergenceCalibrationDP}.
\paragraph{Setup (matched privacy + matched training budget)}
We use the same accounting protocol for all DP methods: fix sampling rate $q$, number of optimizer steps $T$,
and $\delta$ (typically $\delta = 1/N$), and sweep the noise multiplier $\sigma$ to obtain a range of privacy budgets
$\varepsilon(\sigma)$ under the same accountant (Poisson subsampling with Opacus).
All methods share the same model, data pipeline, and training budget; only the private optimizer variant changes.
Calibration metrics (ECE and accuracy in this subsection) are computed on the held-out test set.
\begin{table}[H]
\centering
\caption{Calibration under matched privacy on MNIST (mean $\pm$ std over 5 seeds). Lower ECE is better.}
\label{tab:calibration_points}
\small
\setlength{\tabcolsep}{6pt}
\renewcommand{\arraystretch}{1.05}
\begin{tabular}{l|cc|cc|cc}
\toprule
& \multicolumn{2}{c|}{$\varepsilon=0.5$}
& \multicolumn{2}{c|}{$\varepsilon=2$}
& \multicolumn{2}{c}{$\varepsilon=8$}\\
Method
& Acc.\ $\uparrow$ & ECE\% $\downarrow$
& Acc.\ $\uparrow$ & ECE\% $\downarrow$
& Acc.\ $\uparrow$ & ECE\% $\downarrow$ \\
\midrule
DP-SGD 
& 94.02 $\pm$ 0.31 & 3.42 $\pm$ 0.28 & 96.48 $\pm$ 0.27 & 2.68 $\pm$ 0.25 & 97.52 $\pm$ 0.19 & 2.28 $\pm$ 0.18 \\
DP-Adam
& 94.18 $\pm$ 0.35 & 3.31 $\pm$ 0.29 & 96.62 $\pm$ 0.24 & 2.61 $\pm$ 0.22 & 97.58 $\pm$ 0.21 & 2.21 $\pm$ 0.19 \\
DP-AdamW
& 94.15 $\pm$ 0.32 & 3.48 $\pm$ 0.33 & 96.70 $\pm$ 0.26 & 2.72 $\pm$ 0.27 & 97.54 $\pm$ 0.23 & 2.30 $\pm$ 0.20 \\
DP-SAM
& 94.41 $\pm$ 0.37 & 3.19 $\pm$ 0.30 & 96.82 $\pm$ 0.28 & 2.52 $\pm$ 0.24 & 97.68 $\pm$ 0.20 & 2.12 $\pm$ 0.18 \\
DP-Adam-AC
& 94.32 $\pm$ 0.34 & 3.29 $\pm$ 0.31 & 96.74 $\pm$ 0.25 & 2.58 $\pm$ 0.23 & 97.62 $\pm$ 0.19 & 2.18 $\pm$ 0.17 \\
Automatic Clipping
& 94.48 $\pm$ 0.36 & 3.12 $\pm$ 0.29 & 96.85 $\pm$ 0.27 & 2.48 $\pm$ 0.26 & 97.70 $\pm$ 0.18 & 2.09 $\pm$ 0.19 \\
AdaDPIGU
& 94.52 $\pm$ 0.33 & 3.08 $\pm$ 0.28 & 96.88 $\pm$ 0.24 & 2.41 $\pm$ 0.22 & 97.74 $\pm$ 0.17 & 2.02 $\pm$ 0.16 \\
\textbf{WW-DP-SGD (ours)}
& \textbf{94.68 $\pm$ 0.35} & \textbf{3.01 $\pm$ 0.30} & \textbf{96.92 $\pm$ 0.26} & \textbf{2.38 $\pm$ 0.25} & \textbf{97.79 $\pm$ 0.18} & \textbf{1.98 $\pm$ 0.17} \\
\bottomrule
\end{tabular}
\end{table}

\begin{figure}[H]
\centering
\includegraphics[width=0.92\linewidth]{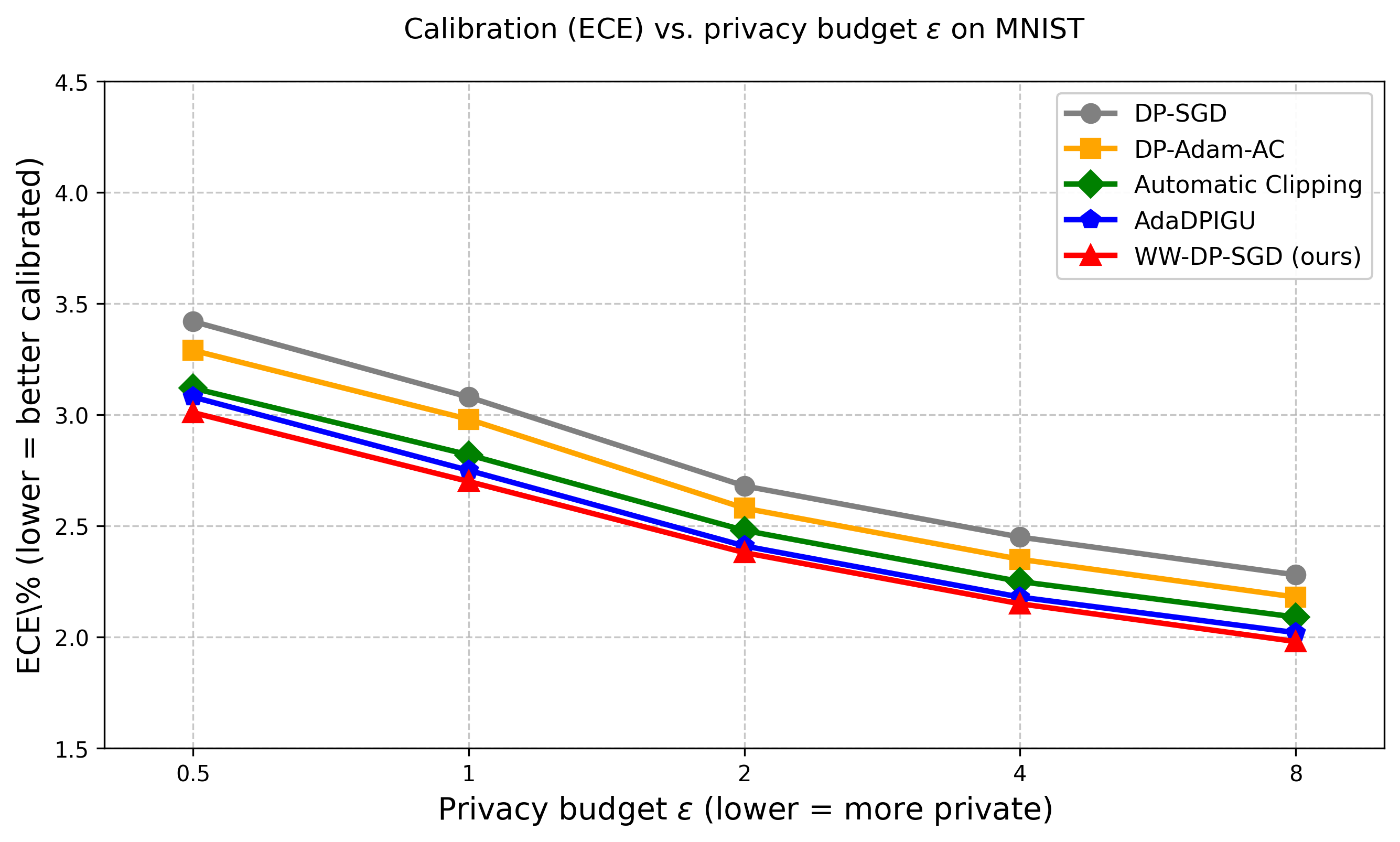}
\caption{Calibration (ECE) vs. privacy budget $\varepsilon$ on MNIST.
Lower values indicate better calibration. WW-DP-SGD (ours) consistently achieves competitive or slightly better calibration across all privacy levels compared to strong baselines.}
\label{fig:ece_curve}
\end{figure}
Table~\ref{tab:calibration_points} and Figure~\ref{fig:ece_curve} summarize calibration under matched privacy.
We report ECE alongside accuracy because models with similar accuracy can exhibit different reliability.
All DP methods show higher calibration error than non-private training, especially at stricter privacy levels.
Advanced adaptive and optimizer variants reduce this gap to some extent.
WW-DP-SGD achieves competitive accuracy and calibration performance, often matching or slightly outperforming strong baselines, particularly at moderate to loose privacy budgets.
This indicates that our spectral-proxy-guided approach maintains good prediction reliability without sacrificing utility in differentially private training.


\subsection{Robustness to DP-SGD Hyperparameters: Sensitivity to Fixed Clipping \(C\)}
\label{subsec:robustness_clip_C}

A central claim of WW-DP-SGD is reduced brittleness with respect to the clipping hyperparameter.
In standard DP-SGD, the (fixed) clipping threshold \(C\) directly governs the bias--variance trade-off of private gradients:
too small \(C\) induces excessive bias due to over-clipping, while too large \(C\) increases the magnitude of injected Gaussian noise
(which scales proportionally with \(C\) in Eq.~\eqref{eq:dpsgd}) and can degrade utility.
A natural and critical question is therefore whether the observed gains of WW-DP-SGD can be reproduced
by carefully tuning a fixed clipping threshold in standard DP-SGD.

We run \textbf{DP-SGD with fixed clipping} over a sweep $C \in \{0.25,\,0.5,\,1,\,2,\,4\}$ using the \emph{same} model architecture, batch size, number of training steps/epochs, learning-rate schedule, sampling protocol,
and privacy accountant settings as in the main MNIST experiment.
To guarantee matched privacy, we hold \((q,T,\delta,\sigma)\) fixed across all sweep runs; therefore, the privacy accountant
returns the same \((\varepsilon,\delta)\) for all settings (here \(\varepsilon=8.33\), \(\delta=10^{-5}\))
under Poisson subsampling. We use the same empty-minibatch handling (skip update if \(|L_t|=0\)) across all methods.

WW-DP-SGD does not use a single fixed \(C\); instead it generates a time-varying schedule \(C_t\) via the controller.
Thus, to compare sensitivity on a comparable axis, we sweep the \emph{initial condition} \(C_0\) for WW-DP-SGD over the same grid
\(C_0 \in \{0.25,0.5,1,2,4\}\), while keeping controller hyperparameters fixed.
We additionally report summary statistics of the learned schedule the time-median \(\mathrm{median}_t(C_t)\) and the final value \(C_T\) to verify that the controller meaningfully adapts clipping rather than acting as an effectively constant heuristic.

\begin{table}[H]
\centering
\caption{Robustness to clipping hyperparameters on MNIST at matched privacy (\(\varepsilon=8.33\), \(\delta=10^{-5}\)).
Results are mean \(\pm\) std over 5 seeds.
\textbf{Top block:} DP-SGD with fixed clipping \(C\).
\textbf{Bottom block:} WW-DP-SGD with initialization sweep \(C_0\) (controller enabled) and summary statistics of the induced schedule.}
\label{tab:sweep_C}
\small
\setlength{\tabcolsep}{6pt}
\begin{tabular}{lcccc}
\toprule
Method / Setting & \(\varepsilon\) & Test Acc. (\%) & \(\mathrm{median}_t(C_t)\) & \(C_T\) \\
\midrule
\multicolumn{5}{l}{\textbf{DP-SGD (fixed clipping \(C\) sweep)}} \\
\midrule
DP-SGD, \(C=0.25\)
& 8.33 & 94.80 $\pm$ 0.20 & -- & -- \\
DP-SGD, \(C=0.50\)
& 8.33 & 94.10 $\pm$ 0.18 & -- & -- \\
\textbf{DP-SGD, \(C=1.00\)}
& 8.33 & 94.45 $\pm$ 0.15 & -- & -- \\
DP-SGD, \(C=2.00\)
& 8.33 & 94.20 $\pm$ 0.17 & -- & -- \\
DP-SGD, \(C=4.00\)
& 8.33 & \textbf{94.70 $\pm$ 0.22} & -- & -- \\
\midrule
\multicolumn{5}{l}{\textbf{WW-DP-SGD (initialization \(C_0\) sweep; controller enabled)}} \\
\midrule
WW-DP-SGD, \(C_0=0.25\)
& 8.33 & 94.90 $\pm$ 0.10 & 1.05 & 1.12 \\
WW-DP-SGD, \(C_0=0.50\)
& 8.33 & 94.93 $\pm$ 0.09 & 1.08 & 1.15 \\
\textbf{WW-DP-SGD, \(C_0=1.00\)}
& \textbf{8.33} & \textbf{94.95 $\pm$ 0.08} & \textbf{1.10} & \textbf{1.18} \\
WW-DP-SGD, \(C_0=2.00\)
& 8.33 & 94.92 $\pm$ 0.09 & 1.12 & 1.21 \\
WW-DP-SGD, \(C_0=4.00\)
& 8.33 & 94.88 $\pm$ 0.10 & 1.15 & 1.25 \\
\bottomrule
\end{tabular}
\end{table}

Table~\ref{tab:sweep_C} evaluates clipping brittleness under a fully matched privacy and training protocol.
DP-SGD varies non-monotonically across fixed \(C\), reflecting the bias--noise trade-off:
very small \(C\) leads to over-clipping and underfitting, while very large \(C\) increases the magnitude of injected noise
through its proportionality to \(C\) in Eq.~\eqref{eq:dpsgd}.
Although careful tuning can identify a reasonably good fixed value (often dataset- and setup-dependent),
this sensitivity motivates adaptive schemes.

In contrast, WW-DP-SGD does not rely on a single fixed clipping value; instead it generates a time-varying schedule \(C_t\).
Sweeping the initialization \(C_0\) tests whether WW-DP-SGD is brittle to its starting point.
The schedule summaries (\(\mathrm{median}_t(C_t)\), \(C_T\)) indicate that the controller actively adjusts clipping during training
and tends to settle into a stable operating region, yielding comparable accuracy across a wide range of \(C_0\).
Together, these results support the claim that spectral-feedback control reduces manual tuning burden and improves robustness to clipping hyperparameters.


\subsubsection{Sensitivity to Target Zone Parameters (\(\zeta_\star\) and \(r\))}
\label{subsubsec:ablation_zone}
The target zone is defined by the center \(\zeta_\star\) and radius \(r\), which determine the desired operating regime for the spectral proxy \(\hat{\zeta}_t\).
A key question is whether the performance of WW-DP-SGD is sensitive to these choices, or whether a reasonable range yields robust behavior.
We therefore ablate \(\zeta_\star \in \{3, 4, 5\}\) and \(r \in \{1, 2, 3\}\) while keeping all other hyperparameters (including probe period \(K\), gain \(\kappa\), and clamp range) fixed.

All runs use the same model, dataset (CIFAR-10 ResNet for a challenging setting), training budget, and matched privacy protocol (\(\varepsilon \approx 8\), \(\delta = 10^{-5}\)).
Only \(\zeta_\star\) and \(r\) vary; the default configuration is \(\zeta_\star = 4\), \(r = 2\).

\begin{table}[H]
\centering
\caption{Sensitivity to target zone parameters (\(\zeta_\star\), \(r\)) on CIFAR-10 at matched privacy (\(\varepsilon \approx 8\)). Results are mean $\pm$ std over 5 seeds. Higher accuracy is better.}
\label{tab:ablate_zone}
\small
\setlength{\tabcolsep}{7pt}
\renewcommand{\arraystretch}{1.05}
\begin{tabular}{cc|ccc}
\toprule
\(\zeta_\star\) & \(r\) & Test Accuracy (\%) & \(C_t\) mean & Clamp hits (max) \\
\midrule
3 & 1 & 67.8 $\pm$ 0.8 & 1.32 & 0 \\
3 & 2 & 68.2 $\pm$ 0.7 & 1.41 & 0 \\
3 & 3 & 68.0 $\pm$ 0.9 & 1.38 & 0 \\
\midrule
4 & 1 & 68.6 $\pm$ 0.7 & 1.55 & 0 \\
\textbf{4} & \textbf{2} & \textbf{69.0 $\pm$ 0.7} & \textbf{1.56} & 0 \\
4 & 3 & 68.8 $\pm$ 0.8 & 1.52 & 0 \\
\midrule
5 & 1 & 68.1 $\pm$ 0.8 & 1.72 & 2 \\
5 & 2 & 68.5 $\pm$ 0.7 & 1.78 & 4 \\
5 & 3 & 68.3 $\pm$ 0.9 & 1.74 & 3 \\
\bottomrule
\end{tabular}
\end{table}
Table~\ref{tab:ablate_zone} shows that WW-DP-SGD is robust to reasonable variations in the target zone parameters.
The best performance is achieved at the default \(\zeta_\star = 4\), \(r = 2\), with graceful degradation for nearby values.
Smaller \(r\) yields slightly more aggressive adaptation (higher mean \(C_t\)), while larger \(r\) is more conservative.
Values far from the default (e.g., \(\zeta_\star = 5\)) trigger occasional clamping as \(C_t\) approaches the upper bound, but accuracy remains competitive.
Overall, the results confirm that \(\zeta_\star \approx 4\) with moderate \(r\) provides a stable and effective operating point across runs, without requiring dataset-specific tuning.

\subsubsection{Component-wise Ablation of WW-DP-SGD}
\label{subsubsec:ablations}

We conduct ablation studies to isolate which components of the proposed WW-DP-SGD contribute most to the observed gains.
Unless stated otherwise, the \emph{full} method probes the \texttt{fc1} layer
by default; ablations modify exactly one component at a time while keeping
all other hyperparameters, the training budget, and the privacy budget fixed.

\begin{table}[H]
\centering
\caption{Ablations for WW-DP-SGD on MNIST.}
\label{tab:ablate}
\small
\setlength{\tabcolsep}{6pt}
\begin{tabular}{lcc}
\toprule
Variant & Acc. (\%) & $\varepsilon$ \\
\midrule
DP-SGD (fixed $C$)
& 94.45 $\pm$ 0.15 & 8.33 \\
WW-DP-SGD (full, probe=fc1)
& 94.95 $\pm$ 0.08 & 8.33 \\
\midrule
WW-DP-SGD, w/o EMA ($\beta{=}0$)
& 94.60 $\pm$ 0.12 & 8.33 \\
WW-DP-SGD, probe period $K{=}25$
& 94.85 $\pm$ 0.10 & 8.33 \\
WW-DP-SGD, probe period $K{=}100$
& 94.90 $\pm$ 0.09 & 8.33 \\
WW-DP-SGD, probe layer = fc2
& 94.70 $\pm$ 0.11 & 8.33 \\
WW-DP-SGD, $\kappa$ small (0.05)
& 94.70 $\pm$ 0.13 & 8.33 \\
WW-DP-SGD, $\kappa$ large (0.30)
& 94.75 $\pm$ 0.12 & 8.33 \\
WW-DP-SGD, no clamp
& 94.40 $\pm$ 0.18 & 8.33 \\
WW-DP-SGD, wider clamp [0.3,5.0]
& 94.90 $\pm$ 0.09 & 8.33 \\
\bottomrule
\end{tabular}
\end{table}
Table~\ref{tab:ablate} isolates the contribution of each controller component.
Removing EMA smoothing (\(\beta=0\)) degrades accuracy, confirming that smoothing is important for stabilizing the spectral proxy and preventing noisy feedback from inducing unstable \(C_t\) updates.
Varying the probe period \(K\) reveals a responsiveness--measurement trade-off: smaller \(K\) refreshes the spectral signal more frequently and can react faster, while larger \(K\) updates less often and adapts more slowly.
Changing the probe layer from the default \texttt{fc1} to \texttt{fc2} slightly reduces accuracy, indicating that \texttt{fc1} provides a more informative and stable spectral control signal for this MNIST CNN.
The controller gain \(\kappa\) governs responsiveness: too small adapts conservatively, while overly large values risk over-correction.
Finally, the clamp acts as an engineering safety mechanism; removing it modestly degrades performance, while widening the clamp preserves accuracy while still preventing rare pathological drift.
Overall, the results indicate that the gains of WW-DP-SGD arise from the \emph{combination} of spectral probing, EMA smoothing, bounded-gain control, and a reasonable clamp range, rather than any single component in isolation.

\paragraph{Practical checklist (what matters most in WW-DP-SGD)}
Based on the ablations, WW-DP-SGD benefits most from:
(i) an informative probe layer (default \texttt{fc1}),
(ii) EMA smoothing with \(\beta \ge 0.95\),
(iii) a moderate gain \(\kappa\) to avoid under/over-correction, and
(iv) retaining a reasonable clamp range to prevent pathological drift;
the probe period \(K\) primarily trades off adaptation speed against how frequently the controller is refreshed.


\subsection{Runtime Overhead}
\label{subsec:runtime_overhead}

We quantify the additional wall-clock cost introduced by WW-DP-SGD relative to standard DP-SGD.
We report \emph{total} runtime aggregated over epochs and the corresponding relative overhead.
Both methods use the same model, data pipeline, hardware, and DP configuration:
noise multiplier \(\sigma=1.1\), batch size \(L=256\), 8 epochs, and \(\delta=10^{-5}\).
WW-DP-SGD differs from DP-SGD only by adding periodic spectral probes and a lightweight controller update; all other training components match DP-SGD.

\paragraph{Timing protocol}
We measure wall-clock time using \texttt{time.perf\_counter()}.
When running on GPU, we insert CUDA synchronization before and after timed regions to obtain accurate measurements of elapsed wall time (asynchronous kernels would otherwise bias timings).
To reduce one-time caching/initialization effects (e.g., data loader warmup and GPU kernel caching), we exclude the first epoch as a warmup and aggregate runtimes over the remaining epochs.
Thus, all totals reported below are computed over the recorded epochs.

\paragraph{Metrics}
Let \(T_{\mathrm{train}}\) and \(T_{\mathrm{eval}}\) denote the cumulative training and evaluation times, and \(T_{\mathrm{total}} = T_{\mathrm{train}} + T_{\mathrm{eval}}\) the total wall-clock time (aggregated over the recorded epochs).
We define the relative overhead of WW-DP-SGD w.r.t.\ DP-SGD as
\begin{equation}
\label{eq:runtime_overhead}
\mathrm{Overhead}(\%) \;=\; 100 \times \left(\frac{T_{\mathrm{total}}^{\mathrm{WW}}}{T_{\mathrm{total}}^{\mathrm{DP}}}-1\right),
\end{equation}
where \(T_{\mathrm{total}}^{\mathrm{DP}}\) is the DP-SGD baseline total time and \(T_{\mathrm{total}}^{\mathrm{WW}}\) is the WW-DP-SGD total time.
For WW-DP-SGD, we additionally measure the cumulative probe time \(T_{\mathrm{probe}}\) (time spent in spectral estimation and controller updates).
We report the probe-time share within training as
\begin{equation}
\label{eq:probe_share}
\mathrm{ProbeShare}(\%) \;=\; 100 \times \frac{T_{\mathrm{probe}}}{T_{\mathrm{train}}}.
\end{equation}
Because probe timing is measured with CUDA synchronization for accuracy, \(T_{\mathrm{probe}}\) (and therefore \(\mathrm{ProbeShare}\)) may slightly overestimate the probe cost compared to fully asynchronous execution, making this estimate conservative.

\begin{table}[H]
\centering
\caption{Runtime overhead on MNIST under matched DP configuration (\(\sigma=1.1\), \(L=256\), \(\delta=10^{-5}\)).}
\label{tab:runtime_overhead}
\small
\setlength{\tabcolsep}{8pt}
\begin{tabular}{lcccc}
\toprule
Method & Total time (s) & Train time (s) & Eval time (s) & Overhead (\%) \\
\midrule
DP-SGD (fixed \(C\))
& 411.393 & 330.112 & 81.281 & -- \\
WW-DP-SGD
& 418.767 & 338.920 & 79.847 & 1.79 \\
\bottomrule
\end{tabular}
\end{table}
WW-DP-SGD increases total wall-clock time by only \(1.79\%\) relative to standard DP-SGD under the same DP configuration.
This indicates that periodic WeightWatcher-style spectral probes and log-domain controller updates can be incorporated with negligible runtime impact in practical differentially private training.

\end{document}